\documentclass{article}

\usepackage[final]{neurips_2021}

\usepackage[utf8]{inputenc} 
\usepackage[T1]{fontenc}    
\usepackage{hyperref}       
\usepackage{url}            
\usepackage{booktabs}       
\usepackage{amsfonts}       
\usepackage{nicefrac}       
\usepackage{microtype}      
\usepackage{xcolor}         

\usepackage{algorithm}
\usepackage{algpseudocode}
\usepackage{amsmath,amssymb,amsthm}
\usepackage{graphicx}
\usepackage{subfigure}
\usepackage{dsfont}
\usepackage{enumitem}
\usepackage{makecell}
\usepackage{caption}
\usepackage{multirow}
\usepackage{wrapfig}

\newtheorem{theorem}{Theorem}
\newtheorem{lemma}{Lemma}

\newcommand{\Rb}{\mathbb{R}}
\newcommand{\Nb}{\mathbb{N}}
\newcommand{\Eb}{\mathbb{E}}
\newcommand{\Pb}{\mathbb{P}}

\newcommand{\lb}{\mathds{1}}

\newcommand{\Ec}{\mathcal{E}}
\newcommand{\Gc}{\mathcal{G}}
\newcommand{\Hc}{\mathcal{H}}

\newcommand{\Af}{\mathfrak{A}}
\newcommand{\Bf}{\mathfrak{B}}
\newcommand{\Cf}{\mathfrak{C}}

\DeclareMathOperator*{\argmax}{arg\,max}
\DeclareMathOperator*{\argmin}{arg\,min}

\allowdisplaybreaks

\ifodd 1

\else

\fi

\ifodd 1
\newcommand{\congc}[1]{{\color{teal}(Cong: #1)}}
\else
\newcommand{\congc}[1]{}
\fi
 
\ifodd 1
\newcommand{\hf}[1]{{\color{cyan}[Haifeng: #1]}}
\else
\newcommand{\hf}[1]{}
\fi

\ifodd 1

\else

\fi

\ifodd 0
\newcommand{\shir}[1]{{\color{blue}#1}}
\else
\newcommand{\shir}[1]{#1}
\fi

\setlist[itemize]{noitemsep, topsep=0pt}

\title{(Almost) Free Incentivized Exploration from Decentralized Learning Agents}

\author{%
    Chengshuai~Shi \\
    University of Virginia\\
    \texttt{cs7ync@virginia.edu} \\
    \And
    Haifeng Xu \\
    University of Virginia \\
    \texttt{hx4ad@virginia.edu} \\
    \And
    Wei Xiong \\
    The Hong Kong University of Science and Technology\\
    \texttt{wxiongae@connect.ust.hk} \\
    \And
    Cong Shen \\
    University of Virginia \\
    \texttt{cong@virginia.edu} \\
}

\begin{document}

\maketitle

\begin{abstract}
Incentivized exploration in multi-armed bandits (MAB) has witnessed increasing interests and many progresses in recent years, where a principal offers bonuses to agents to do explorations on her behalf. However, almost all existing studies are confined to temporary myopic agents. In this work, we break this barrier and study incentivized exploration with multiple and long-term strategic agents, who have more complicated behaviors that often appear in real-world applications. An important observation of this work is that strategic agents' intrinsic needs of learning benefit (instead of harming) the principal's explorations by providing ``free pulls''. Moreover, it turns out that increasing the population of agents significantly lowers the principal's burden of incentivizing. The key and somewhat surprising insight revealed from our results is that when there are sufficiently many learning agents involved, the exploration process of the principal can be (almost) free. Our main results are built upon three novel components which may be of independent interest: (1) a simple yet provably effective incentive-provision strategy; (2) a carefully crafted best arm identification algorithm for rewards aggregated under unequal confidences; (3) a high-probability finite-time lower bound of UCB algorithms. Experimental results are provided to complement the theoretical analysis.

\end{abstract}

\section{Introduction}
Multi-armed bandits (MAB) is a simple yet powerful model for sequential decision making with an exploration-exploitation tradeoff \citep{bubeck2012regret,lattimore2020bandit}.
In standard MAB settings, one principal, who has a long-term system-level objective, takes charge of selecting and playing arms. However, such assumption does not always hold in reality. It is often the case that arm pulls are performed by multiple different agents whose individual goals are not aligned with the system, and the principal can only observe agents' actions. One typical example is the individual buyers (agents) and the online shopping platform (the principal). Such scenarios complicate decision making and introduce significant difficulties to optimize the system performance.

Incentivized exploration has been proposed to address this problem \citep{frazier2014incentivizing,mansour2015bayesian}. Specifically, bonuses can be offered by the principal to incentivize agents to perform specific actions, e.g., to explore their originally underrated arms. {This} framework provides an opportunity to reconcile different interests between the principal and agents. As a concrete example, the online shopping platform can offer discounts on certain items so that individual buyers would buy them and provide feedbacks, which can be used to optimize future strategies of the platform.

While incentivized exploration has been investigated in {the existing literature}, we recognize two major limitations. First, almost all of the existing works assume the participating agents to be \emph{myopic}, i.e., they always choose the empirically best arm. Second, it is always assumed that at each time slot, \emph{one new agent} participates in the system, i.e., agents never stay or return. In other words, prior research mainly considers \emph{how to incentivize one single temporary myopic agent}. These two assumptions largely limit the applicability of incentivized exploration. 

In this work, we extend the study of incentivized exploration beyond the aforementioned barriers, and investigate situations with \emph{multiple long-term strategic agents}. In particular, we focus on the scenarios (see Section~\ref{subsec:motivation}) where the principal wants to identify the (overall)  best arm whereas the   heterogeneously involved agents only care about their different individual cumulative rewards. For such scenarios, the ``Observe-then-Incentivize'' (OTI) 
mechanism is proposed and several interesting observations are obtained. First, we find that strategic agents' intrinsic needs of learning can actually benefit principal's exploration by providing ``free pulls''. In other words, as opposed to myopic agents, the self interests of strategic agents can be \emph{exploited} by the principal. Second, it turns out that increasing the number of participating agents can significantly mitigate the principal's burden on incentivizing, which highlights the importance of increasing the population of agents. A crux of our findings is the following intriguing conceptual message: when there are sufficiently many learning agents involved, the exploration process of the principal could be (almost) free.

Behind these findings, three novel technical components play critical roles in the design and analysis of OTI, all of which may have independent values. 
\begin{itemize}[leftmargin=*]
    \item First, a simple yet provably effective incentive-provision strategy is developed, which can efficiently regulate strategic agents' behaviors and serves as the foundation of the algorithm analysis. 
    \item Second, a best arm identification algorithm is carefully crafted to tackle the varying amounts of local information from heterogeneous agents. This setting itself is novel in best arm identification.
    \item A high-probability finite-time lower bound of UCB algorithms \citep{auer2002finite} is proved, which contributes to a better understanding of the celebrated UCB.
\end{itemize}
These insights and techniques are unique in incentivizing multiple long-term strategic agents, which may find applications in related problems, and encourage future research in this direction.

\section{Related Works}
\textbf{Incentivized exploration.}
Since proposed by \citet{frazier2014incentivizing,kremer2014implementing}, many progresses have been made in incentivized exploration in MAB. Especially, there exist two lines of studies. The first one \citep{kremer2014implementing,mansour2015bayesian,mansour2016bayesian,immorlica2020incentivizing,sellke2020sample} assumes the principal can observe the full history while the agents cannot, and the principal leverage such information advantage to perform incentivizing.
The second line, which our setting follows, considers a publicly available history while the incentives are done through compensations. This idea is first introduced by \citet{frazier2014incentivizing} and generalized by \citet{han2015incentivizing}, both on Bayesian settings. The non-Bayesian case, as adopted in this work, is first studied by \citet{wang2018multi}, and recently extended by \citet{liu2020incentivized,wang2021incentivizing}.

However, the aforementioned works mainly consider that one new myopic agent enters the system at each time slot and leaves afterward. The only exception is \citet{mansour2016bayesian}, where multiple but still temporary and myopic agents are considered. This work differs from them in considering \emph{multiple long-term strategic agents}. In addition, another notable difference is that almost all prior works focus on regret minimization for the principal, instead of best arm identification.

One important related work is \citet{chen2018incentivizing}, which studies temporary myopic agents with heterogeneous preferences. Free explorations are also observed there because heterogeneous preferences result in agents exploiting all arms. However, the ``free pulls'' of OTI is provided by strategic agents' intrinsic needs for explorations, which is fundamentally different. Regardless of the differences, both results show the value of further investigating agents' behaviors in incentivized exploration.

\textbf{Federated MAB.} This work is related to and can potentially contribute to the studies of federated MAB (FMAB) \citep{shi2021federated,shi2021federated2,zhu2020federated}, which considers a similar framework of multiple heterogeneous agents and a global principal. These studies assume agents would unconditionally give up learning their own local ones and naively follow global instructions. which seldom holds in practice and makes those approaches unrobust. However, the proposed incentivizing exploration scheme achieves the ``best'' of both worlds, i.e., agents learn local models with additional compensations and the principal learn the global model via a small amount of cost.

\section{Incentivized Exploration from Decentralized Learning Agents}\label{sec:formulation}

\subsection{Motivation}
\label{subsec:motivation}
Consider the following scenario: one company (the principal) can manufacture several products, and it would like to identify one product that best suits the market. A natural strategy is to perform a market survey by having a group of users to try these products for some time and observing their feedbacks. However, different users have different preferences over these products and they are also learning in this process. Once they identify their preferred products, there is little incentive for them to explore others, which limits the information gathering by the company.  Such a scenario is common in real life. For example,  telecommunications companies such as AT\&T (the principal) try to find the optimal channel to serve the clients (agents) of an area through a period of trials; content-providing websites such as YouTube (the principal) test the proposed features with prospective users (agents).

The idea of incentivized exploration fits in these practical scenarios perfectly since it can be utilized to provide extra bonuses for users to try different products. However, the principal now faces the problem of \emph{how to incentivize multiple long-term strategic agents simultaneously}, which has not been investigated in the prior research to the best of our knowledge. In addition, similar settings of applications are also considered in the recently proposed federated MAB (FMAB) \citep{shi2021federated,zhu2020federated}. However, current FMAB studies only consider naive users who always follow the principal's instructions, and cannot be applied to strategic users that have individual learning abilities and objectives.

\subsection{Agents and the Principal}
Following the motivation example in Section~\ref{subsec:motivation}, we consider a total of $M$ available decentralized agents, each of which interacts with a local bandit environment equipped with the same set of $K$ arms (referred to as \emph{local} arms) but with agent-dependent arm utilities. Namely, at time step $t$, a reward  $X_{k,m}(t)\in[0,1]$ is associated with agent $m$'s action of pulling arm $k$, which is sampled independently from an unknown distribution whose expectation is $\mu_{k,m} := \Eb[X_{k,m}(t)]$. In general, $\mu_{k,m}\neq \mu_{k,n}$ for $m\neq n$. These agents are (naturally) assumed to be decentralized and self-interested, i.e., their goal is to collect as much (of their own) reward as possible during a certain time horizon.

Besides these heterogeneous agents, there is also a \emph{principal}. The principal faces a global bandit game, which also has the same set of $K$ arms (referred to as \emph{global} arms for distinction). The expected rewards of global arm $k$ are the exact average of corresponding local arms, i.e., $\mu_{k} := \frac{1}{M}\sum_{m \in [M]}\mu_{k,m}$. The principal's goal is to identify the optimal global arm with a certain confidence $1-\delta$ within  a time horizon $T$, where $\delta\in (0,1)$ is a pre-fixed constant of the failure probability.

Without loss of generality, each local bandit game is assumed to have one unique optimal local arm, which is denoted as $k_{*,m} := \argmax_{k\in[K]}\mu_{k,m}$ for agent $m$ and its expected reward is $\mu_{*,m}:=\mu_{k_{*,m},m}$. Correspondingly, the sub-optimal gap for arm $k\neq k_{*,m}$ is defined as $\Delta_{k,m} := \mu_{*,m}-\mu_{k,m}$ and for arm $k_{*,m}$, we denote $\Delta_{k_{*,m},m} = \Delta_{\min,m}:=\min_{k\neq k_{*,m}}\Delta_k$. Similarly, we also assume there is one unique optimal global arm $k_* :=\argmax_{k\in[K]}\mu_k$ in the global bandit game, whose expected reward is $\mu_*:=\mu_{k_*}$. The  global sub-optimal gap for arm $k\neq k_*$ is defined as $\Delta_k := \mu_*-\mu_k$ and for arm $k_*$, we denote $\Delta_{k_*} = \Delta_{\min} := \min_{k\neq k_*}\Delta_k$. In addition, the time horizon $T$ is assumed to be known to the principal but not to agents, as it is often the principal who determines such a horizon.

\subsection{Interaction and Observation Model}\label{subsec:observation}
While the agents can directly pull their local arms and observe rewards, the principal cannot interact with the global game.
Instead, she can only observe local actions and rewards. In other words, with agent $m$ pulling arm $k$ at time $t$ and getting reward $X_{k,m}(t)$, the principal can also observe the action $k$ and the reward $X_{k,m}(t)$, which may be used to estimate expected rewards of global arms. It is helpful to interpret the global game as a virtual global characterization of all the local games (although it does not necessarily align with any specific one), which results in the challenge that it cannot be directly interacted with but can only be inferred indirectly. 

This indirect information gathering introduces significant difficulties for the principal. In particular, it is challenging to obtain sufficient local information from self-interested agents to aggregate precise global information. For example, an arm $k$ can be exactly, or very close to, the optimal one in the global model, but also be highly sub-optimal on agent $m$'s local model. Thus, the principal needs large amount of local explorations on arm $k$ to estimate it, which contradicts with agent $m$'s willingness as arm $k$ may not provide high local rewards. Thus, the task of best (global) arm identification is likely to fail only with this passive gathering of information, which is numerically verified in Section~\ref{sec:exp}.

\subsection{Incentivized Exploration}
\begin{wrapfigure}{R}{0.45\textwidth}	
	\centering
	\includegraphics[width=0.45\textwidth]{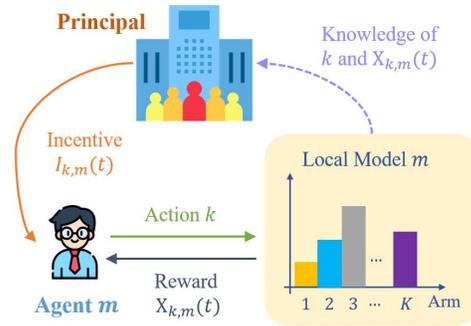}
	\caption{Incentivized exploration of the principal and agent $m$. The agent performs actions and gets both rewards and incentives.  The principal pays the incentives and observes local actions and rewards.}
	\label{fig:framework}
	\vspace{-0.1in}
\end{wrapfigure}

To address these challenges, we resort to the paradigm of \emph{incentivized exploration}, which provides a means for the principal to influence local actions. At time step $t$, the principal can provide extra bonuses for agents to explore arms, which will be announced to the agents before their decision making. Specifically, at time $t$, the bonus on arm $k$ for agent $m$ is denoted as $I_{k,m}(t)\geq 0$. If agent $m$ chooses arm $k$, she has an observation of $X_{k,m}(t)$ but obtains a reward $X'_{k,m}(t) = X_{k,m}(t)+I_{k,m}(t)$. Intuitively, if the principal wants to incentivize agent $m$ to explore a certain arm $k$ against the agent's original willingness, she should give a high bonus, i.e., $I_{k,m}(t)>0$; otherwise, she should spare no bonus, i.e., $I_{k,m}(t)=0$. In this way, the principal can leverage the extra bonuses to have agents gather her desired local information. The interaction model with incentivized exploration between the principal and one agent (among overall $M$ agents) is illustrated in Fig~\ref{fig:framework}.

Under the basic framework, we now formally define the learning objectives of agents and the principal.

{\bf Agents' Objectives. } First, the self-interested agents want to collect as many rewards as possible, but note that the rewards now consist of two parts: the original rewards generated by pulled arms and additional bonuses from principal's incentives. Thus, the cumulative rewards of agent $m$ is defined as 
\begin{equation*}
    R_m(T) := \sum\nolimits_{t=1}^T \left(X_{\pi_m(t),m}(t)+I_{\pi_m(t),m}(t)\right)
\end{equation*}
where $\pi_m(t)$ denotes the arm pulled by agent $m$ at time $t$. 
We remark that such rational learning agents differ fundamentally from myopic ones assumed  in most research of incentivized exploration.

{\bf Principal' Objectives. }On the principal's side, her first goal is to identify the best global arm with a confidence higher than $1-\delta$. Rigorously, this goal can be stated as 
\begin{equation*}
    \Pb\left[\hat{k}_*(T) = k_*\right]\geq 1-\delta,
\end{equation*}
where $\hat{k}_*(T)$ denotes the identified arm at horizon $T$.

Given that the identification is correct, the principal also aims at spending as few cumulative incentives as possible, which is defined as
\begin{equation*}
    C(T) := \sum\nolimits_{t=1}^T\sum\nolimits_{m\in [M]}I_{\pi_m(t),m}(t).
\end{equation*}
Note that  best arm identification is the principal's major task whereas minimizing cost $C(T)$ is only meaningful given $k_*(T) = k_*$ is achieved. Though not central, our mechanism will also satisfy other desirable properties such as storage efficiency.

In this work, we adopt the perspective of the principal, and try to optimize the principal's performance w.r.t. her learning objectives specified above. 
In other words, agents are ignorant of principal's goals and only focus on their individual rewards, and we only incorporate their ``selfishness'' to design an effective and efficient strategy for the principal. Such choice is natural in common applications where the behaviors of principal (e.g., company, commercial platform) can be designed, while agents (e.g., users, customers) normally perform their own decision making which cannot be specified.

\textbf{Remark.} Note that agents and the principal actually represent two kinds of bandit learning objectives. Namely, the agents target regret minimization \citep{auer2002finite,garivier2011kl} (although we here use the equivalent notation of cumulative rewards instead of regret), while the principal aims at best arm identification \citep{audibert2010best,jamieson2014lil,garivier2016optimal}. This formulation is reasonable as agents often care more about their own cumulative benefits while the principal aims at the final result, which can then be used in the future. However, as noted in \citet{bubeck2011pure}, even with the same game instance, these two objectives do not necessarily align with each other, not to mention the additional global-local heterogeneity considered in our model.

\section{A Mechanism for Incentivizing Exploration}
In this section, the ``Observe-then-Incentivize'' (OTI) algorithm is proposed, which can effectively solve the best arm identification problem on the global model while using a small amount of incentives and maintaining efficient storage. As indicated by the name, the key idea of OTI is ``Observe-then-Incentivize'', which comes from its two phases: the observing phase and the incentivizing phase.

\subsection{The ``Take-or-Ban'' Incentive-Provision Protocol} \label{subsec:take_leave} 
Before designing detailed mechanisms, we first recognize one major challenge to incentivize forward-looking agents is that even if we provide sufficient bonuses to compensate her reward at the current round, the agent may not want to take it for various reasons, e.g., giving up her current choice may lead to significant future losses, or refusing the current compensation may trick the principal to offer more future bonuses. Therefore, it is not clearly how to regulate agents' behaviors through incentives. To overcome this barrier, we propose the ``Take-or-Ban'' incentive-provision protocol, which provably guarantees that it is in the best interest for every agent to follow the offered incentives. This protocol is announced to the agents at the beginning of the game, and detailed as follows.

\textbf{``Take''.} At time step $t$, bonuses $I_{k,m}(t)$ offered to agent $m$ are set as a binary value. Specifically, $I_{k,m}(t) = 1$ if principal wants to incentivize exploration on her arm $k$; otherwise $I_{k,m}(t) = 0$. In other words, if her arm $k$ is incentivized, agent $m$ gets reward $X'_{k,m}(t)=1+X_{k,m}(t)$ by pulling it.
 
\textbf{``Ban''.} To avoid intractable agent behaviors, the following safeguard approach is adopted. Specifically, at time step $s$, if agent $m$ is provided incentive for taking some action (i.e., $\exists k, I_{k,m}(s)=1$) but she does not pull it (i.e., $\pi_m(s)\neq k$), she is marked as ``banned'' by the principal. The principal stops providing bonuses for any banned agent in the future, i.e., $\forall t>s, \forall k\in[K], I_{k,m}(t)  = 0$. In other words, due to her failure in following the current incentive, agent $m$ loses the chance of taking future bonuses (bu she is still free to play the local game and get original rewards $X_{k,m}(t)$).

\subsection{The Observing Phase}
Interestingly, while strategic agents introduce substantial challenges to the principal's learning problem, their fundamental need of learning local models can be a blessing.
In order to collect more local rewards, agents naturally have to address the intrinsic ``exploration-exploitation'' dilemma. In particular, to guarantee low regrets, some (but maybe limited) explorations on each arm are required by each agent. Specifically, this argument is supported by the following asymptotic lower bound \citep{Lai:1985}: with any consistent agent strategy,\footnote{A ``consistent'' strategy has a regret of $o(T^{\psi})$ in any bandit instance, $\forall \psi>0$ \citep{Lai:1985}.} for arm $k\neq k_{*,m}$, it holds that
\begin{equation}\label{eqn:asy_lower}
    \liminf_{\Gamma\to \infty} \frac{\Eb[N^w_{k,m}(\Gamma)]}{\log(\Gamma)} \geq \frac{1}{\texttt{KL}(\mu_{k,m},\mu_{*,m})},
\end{equation}
where $N^w_{k,m}(\Gamma)$ is the number of pulls by agent $m$ on arm $k$ up to time $\Gamma$ when there are no incentives, and $\texttt{KL}(\mu_{k,m},\mu_{*,m})$ is the KL-divergence between two corresponding reward distributions. In other words, to guarantee local performance, agents would spontaneously explore all local arms.

On the principal's side, this observation indicates that free local information can be obtained by just letting agents directly run their own local algorithms. Thus, intuitively, it is wise not to incentivize at the beginning of the game, but to observe and enjoy the ``free rides'' provided by agents' exploration. 

However, since local models are often not aligned with the global model, it may not be desirable to fully rely on the ``free rides'' from local agents' actions since they may gradually converge to their own optimal local arms, which may not be what the principal is interested in exploring.

\begin{wrapfigure}{R}{0.5\textwidth}
\vspace{-0.2in}
\begin{minipage}{0.5\textwidth}
\begin{algorithm}[H]
	\small
	\caption{OTI: Principal}
	\label{alg:principal}
	\begin{algorithmic}[1]
	    \State Initialization: $\forall k\in[K], \forall m\in[M], N_{k,m}(0)\gets 0,\hat{\mu}_{k,m}(0)\gets 0$
        \For{$t = 1, 2, \cdots, \frac{T}{2}$} \Comment{\textit{Observing Phase}}
            \State $\forall k\in[M], \forall m\in[M]$, $I_{k,m}(t)\gets 0$
            \State $\forall m\in[M]$, observe $\{\pi_m(t), X_{\pi_m(t),m}(t)\}$, then update $N_{\pi_m(t),m}(t)$ and $\hat{\mu}_{\pi_m(t),m}(t)$
        \EndFor
        \Statex $\triangleright$ \textit{Incentivizing Phase:}
        \State Set $S(\frac{T}{2})\gets [K]$ \Comment{\textit{Incentivizing Phase}}
         \For{$t = \frac{T}{2}+1, \frac{T}{2}+2, \cdots, T$}
            \State $\forall k\in[K]$, $\hat{\mu}_k(t-1) \gets \frac{1}{M}\sum_{m\in[M]} \hat{\mu}_{k,m}(t-1)$ and  set $CB_k(t-1)$ with Eqn.~\eqref{eqn:confidence}
            \State Update $S(t)$ as specified in Eqn.~\eqref{eqn:elimination}
            \If{$|S(t)|\geq 1$}
            \State $\bar{k}(t)\gets \argmax_{k\in S(t)}CB_k(t-1)$ 
            \State $\bar{m}(t) \gets \argmin_{m\in[M]}N_{\bar{k}(t),m}(t-1)$
            \State $I_{\bar{k}(t),\bar{m}(t)}(t) \gets  1$
            \State $I_{k,m}(t) \gets 0$, $\forall m\neq \bar{m}(t),\forall k\neq \bar{k}(t)$
            \Else 
            \State $\hat{k}_*(T)\gets$ the remaining arm in $S(t)$
            \EndIf
        \EndFor
        \Ensure $\hat{k}_*(T)$
	\end{algorithmic}
	\end{algorithm}
\end{minipage}
\vspace{-0.3in}
\end{wrapfigure}

Thus, it is necessary to reserve some time before the end of the game for adaptive adjustments. By putting these intuitions together, in the design of OTI, we specify the observing phase to last $\kappa(T) = \frac{T}{2}$ time steps from the beginning of the game.\footnote{Note that although $\kappa(T)$ is chosen to be $\frac{T}{2}$ here, there are other possible choices, e.g., $\frac{T}{4}$, $\sqrt{T}$, etc. Details for the influence of this choice are provided in Appendix~\ref{supp:dissucsion}.} To summarize, in this observing phase, the principal does not incentivize at all, i.e., $\forall t\in[0,\frac{T}{2}], \forall k\in [K], \forall m\in [M], I_{k,m}(t) = 0$. Instead, she only observes local actions and rewards. Due to our choice of space-efficient information aggregation (specified at the end of Section~\ref{subsec:incentive}), by the end of the observing phase, the principal has the record $\{N_{k,m}(\frac{T}{2}),\hat{\mu}_{k,m}(\frac{T}{2})|\forall k\in[K], \forall m\in[M]\}$, where $N_{k,m}(t)$ is the number pulls performed by agent $m$ on arm $k$ up to time $t$ and $\hat{\mu}_{k,m}(t)$ is the corresponding sample means from these pulls. These information serve as the foundation for the remaining $\frac{T}{2}$ time slots of the incentivizing phase.

\subsection{The Incentivizing Phase}\label{subsec:incentive}
From time slot $\frac{T}{2}+1$, the principal enters the incentivizing phase, where she actively leverages incentives instead of passively observing.
The first challenge she has is how to aggregate local information. Especially, the local sample means (i.e., $\hat{\mu}_{k,m}(t)$) are from different agents and associated with different number of pulls (i.e., $N_{k,m}(t)$), which further result in their individually unequal uncertainties. To address this challenge, a new (global) arm elimination algorithm is proposed, which is inspired by standard arm elimination algorithms in best arm identification \citep{even2002pac,karnin2013almost} but is specifically designed to tackle the new issue of global-local heterogeneity. 

Specifically, at each time step $t\in[\frac{T}{2}+1,T]$, with the record $\{N_{k,m}(t-1),\hat{\mu}_{k,m}(t-1)|\forall k\in[K], \forall m\in[M]\}$, the principal estimates the expected reward of global arm $k$ as $\hat{\mu}_k(t-1) = \frac{1}{M}\sum_{m\in[M]} \hat{\mu}_{k,m}(t-1)$
and associates global arm $k$ with the following confidence bound: 
\begin{equation}\label{eqn:confidence}
    CB_k(t-1) = \frac{1}{M}\sqrt{\bigg(\sum\nolimits_{m\in[M]}\frac{1}{N_{k,m}(t-1)}\bigg)\bigg(\log\bigg(\frac{KT}{\delta}\bigg)+4M\log\log\bigg(\frac{KT}{\delta}\bigg)\bigg)}.
\end{equation}
Note that Eqn.~\eqref{eqn:confidence} incorporates the different number of pulls on arm $k$ by \emph{all} the local agents, i.e., $\{N_{k,m}(t-1)|\forall m\in[M]\}$. It is more challenging than the standard confidence bound design in best arm identification \citep{jamieson2014best,gabillon2012best}, which only considers one source of arm pulls. This complication can be better understood in later theoretical analysis.

With the estimation and confidence bound, arms that are sub-optimal with high probabilities can be identified and eliminated while the other arms are left for more explorations in the future. Namely, let the active arm set $S(t)$ denote the arms that have not been determined to be sub-optimal up to time step $t$, which is initialized as $S(\frac{T}{2}) = [K]$, the new active arm set $S(t)$ is updated from $S(t-1)$ as
\begin{equation}\label{eqn:elimination}
    S(t) = \left\{k\in S(t-1)|\hat{\mu}_k(t-1)+CB_k(t-1)\geq \max_{j\in S(t-1)}\left\{\hat{\mu}_j(t-1)-CB_j(t-1)\right\}\right\}.
\end{equation}

To have the above-illustrated procedure effectively iterate over time, explorations are required for the arms in set $S(t)$. Hence, incentives play a critical role and the principal needs to decide which arm-agent pair to incentivize. In OTI, this decision process consists of two steps. The first step is to find the active arm with the largest confidence bound, i.e., \begin{equation*}
    \bar{k}(t):=\argmax_{k\in S(t)}CB_k(t-1).
\end{equation*}%
Then, the second step is to further identify the local agent who has the least pulls on arm $\bar{k}(t)$ i.e., 
\begin{equation*}
    \bar{m}(t) := \argmin_{m\in[M]}N_{\bar{k}(t),m}(t-1).
\end{equation*}%
Finally, the principal would only incentive agent $\bar{m}(t)$ to explore the arm $\bar{k}(t)$, i.e., 
\begin{equation*}
	I_{k,m}(t) \gets
	\begin{cases}
		1 & \text{if $k = \bar{k}(t)$ and $m = \bar{m}(t)$}\\
		0 & \text{otherwise}
	\end{cases}.
\end{equation*}
Intuitively, the arm-agent pair $\{\bar{k}(t), \bar{m}(t)\}$ represents the largest source of uncertainties in the current estimation of active global arms, i.e., the arm with largest confidence interval and the agent who had the fewest pulls on it. Thus, explorations on this pair is naturally the most efficient way to increase the confidence of estimations. 

By iterating this process of eliminating and incentivizing, the principal would eventually have sufficient information to shrink the active arm set to have only one arm left when the horizon is sufficient, i.e., $|S(t)|=1$, and this remaining arm is output as the identified optimal arm $\hat{k}_*(T)$.

{\bf Space-efficient Information Aggregation.}
While the principal can observe local actions and rewards, it is not storage-friendly to store all raw local data sequence $\{\pi_m(\tau),X_{\pi_m(\tau),m}(\tau)|\forall m\in[M], \forall \tau\leq t\}$, which requires memory space of order $O(Mt)$ that grows linearly in $t$. Instead, OTI is designed to keep track of $\{N_{k,m}(t),\hat{\mu}_{k,m}(t)|\forall k\in[K], \forall m\in[M]\}$.
These values can be updated iteratively and only take a constant memory space of $O(KM)$ regardless of the horizon. 

\section{Theoretical Analysis}\label{sec:theory}
\subsection{Main Results}\label{subsec:main_result}
The performance of OTI is theoretically analyzed from several different aspects. With agents running general consistent local strategies, the following theorem establishes the success of best arm identification and bounds the \emph{expected} cumulative incentives.
\begin{theorem}\label{thm:expected_incentive} 
It is the best interest for every agent to always accept the incentivized explorations under the ``Take-or-Ban'' protocol. Moreover, if the agents' local strategy is consistent without incentives and the horizon $T$ is sufficiently large, the OTI algorithm satisfies that $\Pb[\hat{k}_*(T) = k_*]\geq 1-\delta$, and the expected cumulative incentives are bounded as
    \begin{equation}\label{eqn:expected_incentive}
        \Eb[C(T)] = O\bigg( \sum_{k\in [K]}\sum_{m\in[M]}\bigg[\frac{\log(\frac{KT}{\delta})}{M\Delta_k^2}+\frac{\log\log(\frac{KT}{\delta})}{\Delta_k^2}- \min\bigg\{\frac{T}{2}, \frac{\log(\frac{T}{2})}{\textup{\texttt{KL}}(\mu_{k,m},\mu_{*,m})}\bigg\}\bigg]^+\bigg),
    \end{equation}
where $x^+ := \max\{x,0\}$.
\end{theorem}

While Theorem~\ref{thm:expected_incentive} provides a general upper bound with arbitrarily consistent local strategies,  Eqn.~\eqref{eqn:expected_incentive} is an upper bound on \emph{expected} cumulative incentives. It does not imply a (stronger) high-probability bound. For example, the incentives may be of order $O(\log^2(T))$ with $1/\log(T)$ probability.

To better understand the performance of OTI, next we consider agents who run  UCB  \citep{auer2002finite},  one of the most commonly adopted MAB algorithms. Specifically, agents are assumed to run the following   $\alpha$-UCB algorithm \citep{bubeck2012regret} when there are no incentives:
\begin{equation*}
    \pi_m(t)\gets \argmax_{k\in[K]}\left\{\hat{\mu}_{k,m}(t-1)+\sqrt{\alpha \log(t)/N_{k,m}(t-1)}\right\},
\end{equation*}
where $\alpha$ is a positive constant specified in the design, and a typical choice is $\alpha = 2$ \citep{auer2002finite}. In this case, we are able to achieve a stronger \emph{high-probability} guarantee for incentives.
\begin{theorem}\label{thm:ucb_incentive}
While the agents run $\alpha$-UCB algorithms with $\alpha\geq \frac{3}{2}$ and the horizon $T$ is sufficiently large, the OTI algorithm satisfies that $\Pb[\hat{k}_*(T) = k_*]\geq 1-\delta$.
Moreover, it holds that
\begin{equation}\label{eqn:ucb_incentive}
    \Pb\bigg[C(T) = O\bigg( \sum_{k\in [K]}\sum_{m\in[M]}\bigg[\frac{\log(\frac{KT}{\delta})}{M\Delta_k^2}+\frac{\log\log(\frac{KT}{\delta})}{\Delta_k^2}- \frac{\alpha\log(\frac{T}{2})}{\Delta_{k,m}^2}\bigg]^+\bigg)\bigg]\geq 1-\frac{4MK}{T}.
\end{equation}
\end{theorem}
In both Eqns.~\eqref{eqn:expected_incentive} and \eqref{eqn:ucb_incentive}, the first two terms represent the number of pulls that the principal needs on agent $m$'s arm $k$ to determine whether it is optimal or not. They are proportional to $1/\Delta^2_k$ as in standard best arm identification algorithms \citep{jamieson2014best}. In addition, the second term is a lower-order one w.r.t. $1/\delta$.  Taking a deeper look into these results, we see that the first term decreases with increasing number of agents (proportional to $1/M$). This observation indicates that increasing the population of agents actually benefits the learning of the principal. We note that the importance of \emph{agent population} has not been fully recognized in prior studies.

Furthermore, both last terms in Eqns.~\eqref{eqn:expected_incentive} and \eqref{eqn:ucb_incentive} characterize the number of spontaneous pulls that agent $m$ performs on arm $k$ during the observing phase, which is the amount of ``free pulls'' taken by the principal. This term is guaranteed by Eqn.~\eqref{eqn:asy_lower} for general consistent local strategies, and by the to-be presented Lemma~\ref{lem:ucb_lower} for UCB. In other words, the learning behavior of strategic agents benefits the exploration of the principal. This observation is interesting since as opposed to most prior studies of myopic agents, the principal can leverage the natural behavior of strategic agents.

Both of the aforementioned observations lead to our key result, that if there are enough amount of agents, i.e., $M$ is large, the last term dominates the first two terms, which means no incentives are needed. Correspondingly, a somewhat surprising result emerges -- when there are sufficiently many learning agents involved, the exploration process of the principal can be (almost) free.

The key parts in the proofs are illustrated in the following, which may be of independent interests.

\subsection{Proof Step 1: Effectiveness of the ``Take-or-Ban'' Incentive-Provision Strategy}\label{subsec:step1}
In prior works with myopic agents, it is obvious that with sufficient instantaneous bonuses, they would pull the incentivized arms. However, this work deals with strategic agents with long-term goals. As stated in Section~\ref{subsec:take_leave}, these strategic agents have much more complicated behaviors, e.g., they may occasionally (instead of always) follow incentives, or they may refuse incentives first but accept later, which requires a more careful agent behavior analysis. Fortunately, our ``Take-or-Ban'' protocol guarantees that agents will always follow the incentives, as stated in the following lemma. 

\begin{lemma}\label{lem:participate}
Under the  ``Take-or-Ban'' incentive-provision protocol, following incentives, whenever offered, is optimal w.r.t. the expected cumulative rewards (compared to not following) for every agent. 
\end{lemma}
Note that the above lemma does not rely on agent's learning algorithm and holds for all possible ones. Thus regardless of local agents' original intentions, as long as they are self-interested and rational, \emph{they should always follow the incentives}, i.e., pull the arm which has extra bonus offered. This design provides a clean characterization  of agent behaviors for our analysis of OTI next. 

\shir{\textbf{Remark.} The ``Take-or-Ban'' protocol is designed for theoretical rigor. For practical applications, the design OTI algorithm can be implemented with relaxed protocols; it is just that the rigorous theoretical incentive guarantee may not hold for some rational users with sophisticated strategies.}

\subsection{Proof Step 2: Effectiveness of Best Arm Identification}\label{subsec:step2}
Some key lemmas are presented to demonstrate the effectiveness and efficiency of the designed best arm identification algorithm.
Inspired by proof techniques in combinatorial MAB (CMAB) literature \citep{combes2015combinatorial}, the confidence bound design in Eqn.~\eqref{eqn:confidence} is validated in the following lemma.
\begin{lemma}\label{lem:confidence}
    Denote $\Hc:=\left\{\forall t\in[\frac{T}{2}+1,T], \forall k\in [K], \left|\hat{\mu}_{k}(t-1)-\mu_k\right|\leq CB_k(t-1)\right\}$.
    When the horizon $T$ is sufficiently large, it holds that $\Pb(\mathcal{H})\geq 1-\delta$.
\end{lemma}

Conditioned on event $\Hc$, the required number of local pulls is characterized in the following lemma.
\begin{lemma}\label{lem:good_event}
    When event $\Hc$ happens, $\forall t\in [\frac{T}{2}+1,T]$, we have $k_*\in S(t)$, i.e., the optimal global arm would not be eliminated. Moreover, it suffices to eliminate arm $k\neq k_*$ at time $t$, i.e., $k\notin S(t)$, when
    \begin{equation*}
    \forall m\in[M], N_{k,m}(t-1), N_{k_*,m}(t-1)\geq \frac{16\log(KT/\delta)}{M\Delta_k^2}+\frac{64\log\log(KT/\delta)}{\Delta_k^2}.
    \end{equation*}
\end{lemma}
With the decision process in Section~\ref{subsec:incentive} and the lower bound in Eqn.~\eqref{eqn:asy_lower}, Theorem~\ref{thm:expected_incentive} can be proved.  

\subsection{Proof Step 3: A Finite-time Lower Bound of UCB}\label{subsec:step3}
An important ingredient to prove Theorem~\ref{thm:ucb_incentive} is a finite-time high-probability lower bound for $\alpha$-UCB:
\begin{lemma}\label{lem:ucb_lower}
    When $\Lambda$ satisfies $\frac{\Lambda}{\log^2(\Lambda)}> \frac{4K(\alpha-3/2)^2}{\Delta_{\min,m}^4}$, the $\alpha$-UCB algorithm with $\alpha\geq \frac{3}{2}$ satisfies that
    \begin{equation}\label{eqn:ucb_lower}
         \Pb\bigg[\forall k\in [K],N^w_{k,m}(\Lambda)\geq \frac{(\sqrt{\alpha}-\sqrt{1.5})^2\log(\frac{\Lambda}{2})}{4\Delta_{k,m}^2}\bigg]\geq 1-\frac{2K}{\Lambda}.
    \end{equation} 
\end{lemma}

We note that Eqn.~\eqref{eqn:ucb_lower} is valuable along multiple lines. First, it is a finite-time bound as opposed to the asymptotic one in  Eqn.~\eqref{eqn:asy_lower}. Second, while Eqn.~\eqref{eqn:asy_lower} holds in expectation, Eqn.~\eqref{eqn:ucb_lower} is a stronger high-probability bound, and implies a bound of the same order for expectation.  We believe that this result itself may contribute to the understanding of UCB. Specifically, this lemma characterizes UCB's conservativeness, i.e., it would (nearly) always explore \emph{every} arm at least logarithmic times.

\section{Experiments}\label{sec:exp}
Numerical experiments have been carried out to evaluate OTI. All the results are averaged over $100$ runs of horizon $T=10^5$ and the agents perform the $\alpha$-UCB algorithm specified in Section~\ref{subsec:main_result} with $\alpha=2$. More experimental details can be found in Appendix~\ref{supp:exp}.

\begin{figure*}[t]
	\setlength{\abovecaptionskip}{-1pt}
	\centering
	\subfigure[With or w/o incentives.]{ \includegraphics[width=0.235\linewidth]{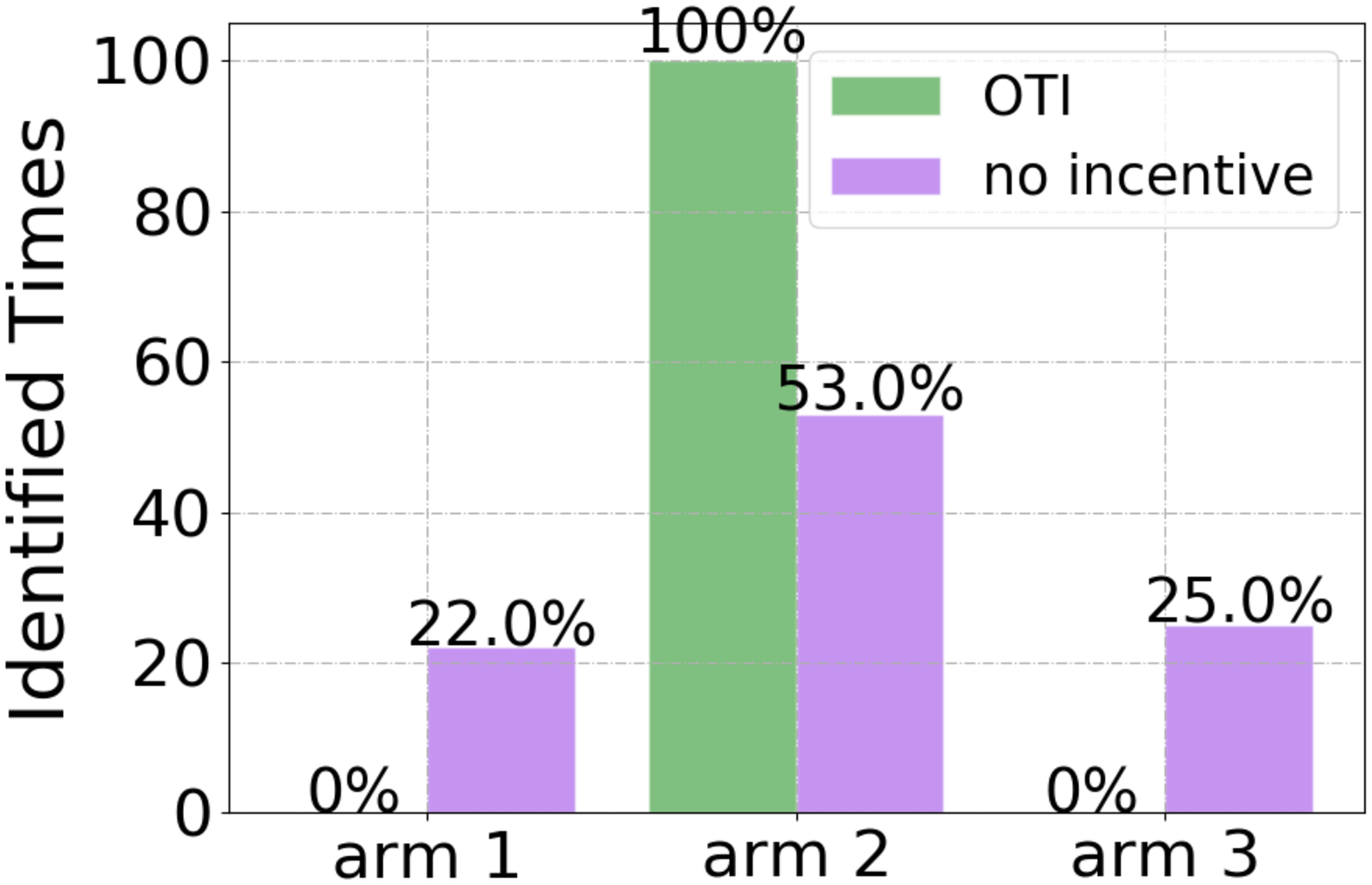}\label{fig:no_incentive}}
	\subfigure[Incentives assignment.]{ \includegraphics[width=0.235\linewidth]{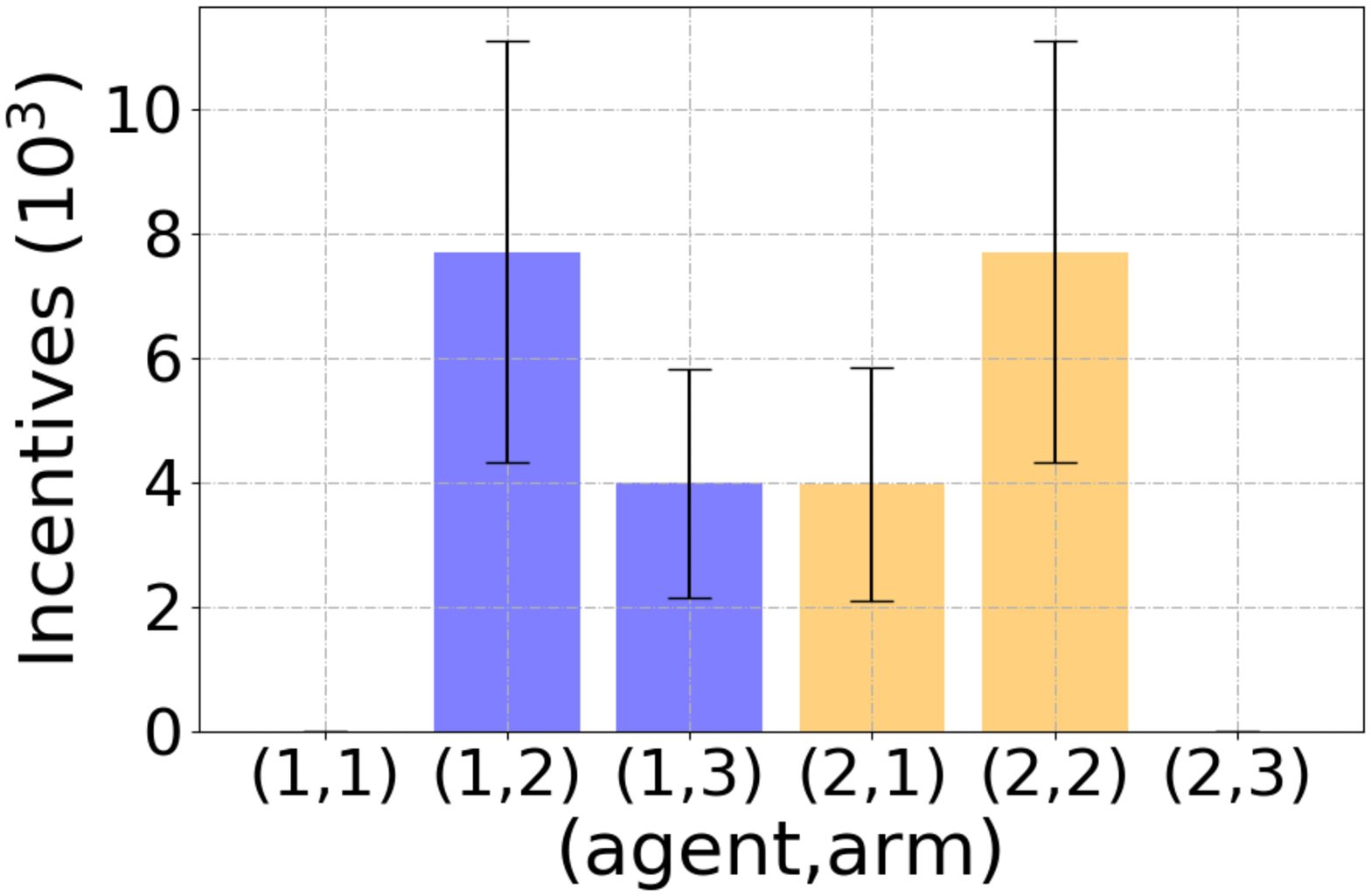}\label{fig:incentive_distribution}}
	\subfigure[Incentives v.s. $\delta$]{ \includegraphics[width=0.235\linewidth]{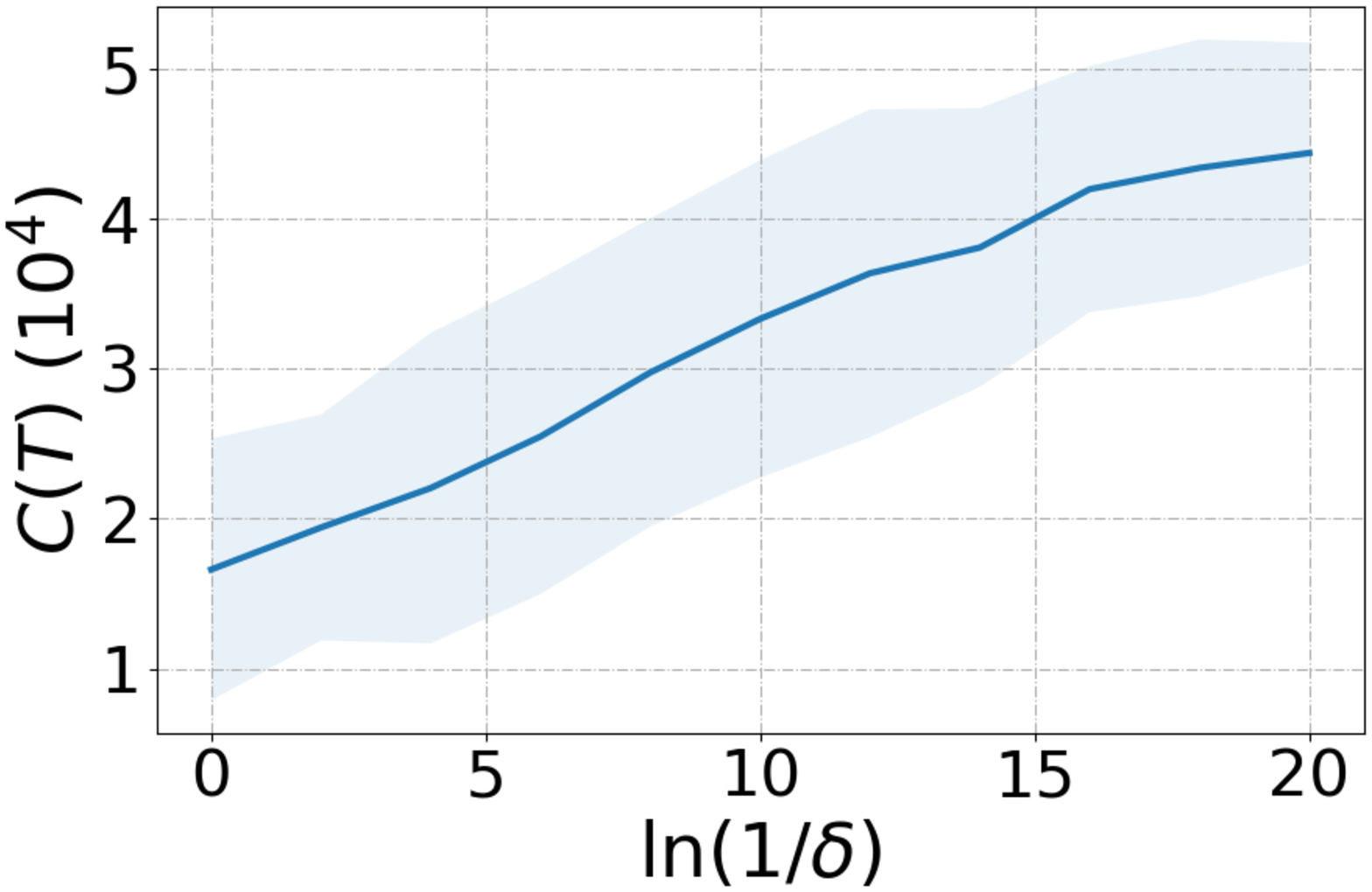}\label{fig:incentive_delta}}
	\subfigure[Incentives v.s. $M$]{ \includegraphics[width=0.235\linewidth]{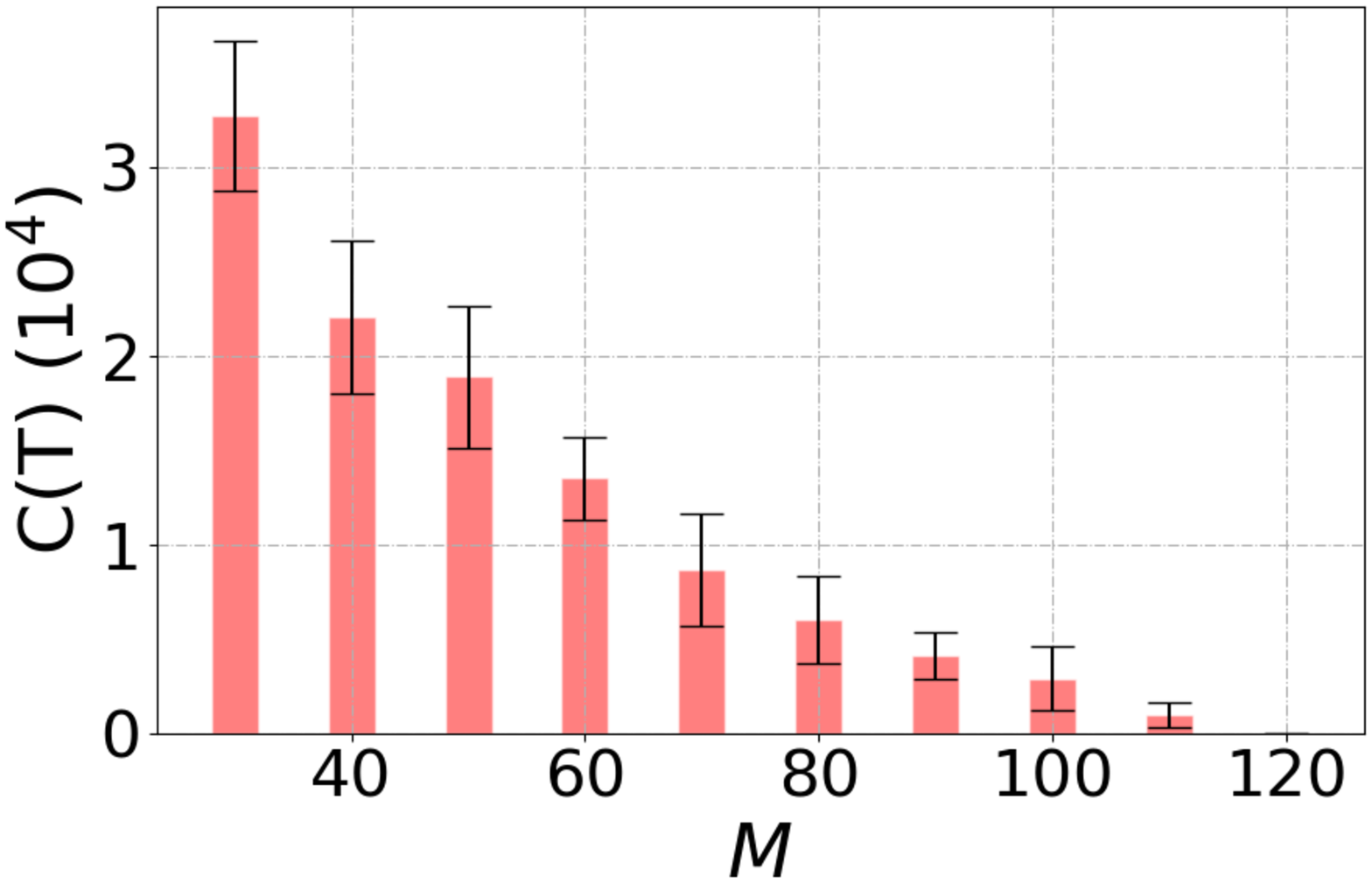}\label{fig:incentive_M}}
	\caption{Experimental results. (a)-(c) are performed under a $2$-agents-$3$-arms example while (d) is evaluated with random instances with $30$ arms and varying number of agents. (a) reports the identification accuracy with and w/o incentives, (b) the assignment of incentives, (c) the logarithmic dependence on $\delta$, and (d) the diminishing effect of cumulative incentives  with $M$ increasing.
	}
	\label{fig:performance}
	\vspace{-0.2in}
\end{figure*}

First, with a toy example of $M=2$ agents and $K=3$ arms, the ineffectiveness of not incentivizing is illustrated. Specifically, agent $1$'s expected rewards for the three arms are set as $[0.89, 0.47, 0.01]$ while agent $2$'s as $[0.01, 0.47, 0.89]$, which results in a global instance with expected rewards $[0.45, 0.47, 0.45]$.\footnote{Although being a toy example, the seemingly simple instance is actually hard in terms of a small global sub-optimality gap ($\Delta_{\min}=0.02$), large global-local divergences and a small number of involving agents.} Note that the optimal global arm is arm $2$, while the local optimal arm is arm $1$ (resp. $3$) for agent $1$ (resp. $2$), which raises the global-local conflicts. Without incentives, the principal can only output the arm with the largest aggregated global mean at the end of the horizon. As shown in Fig.~\ref{fig:no_incentive}, such a ``purely passive'' principal does not perform the identification well. Especially, she only outputs the correct arm (i.e., arm $2$) with $53\%$ accuracy. To make things worse, the principal has no control of this result, which may be even lower in other instances. 

As opposed to the poor performance without incentives, Fig.~\ref{fig:no_incentive} demonstrates that with incentives, OTI (using $\delta=0.01$) can always identify the optimal global arm. Correspondingly, Fig.~\ref{fig:incentive_distribution} presents the amount of incentives spent on each agent-arm pair. It can be observed that the principal never assigns incentives on the optimal arm of each agent, i.e., arm $1$ (resp. $3$) for agent $1$ (resp. $2$), which is intuitive since agents converge to these arms quickly and the ``free pulls'' on them is already sufficient. Furthermore, most of incentives are on arm $2$, which is because it is sub-optimal for both agents and lacks natural explorations. Moreover, OTI is tested with varying $\delta$ under the same $2$-agents-$3$-arms instance. Fig.~\ref{fig:incentive_delta} illustrates that the cumulative incentives of OTI are (nearly) proportional to $\log(1/\delta)$, which verifies the logarithmic dependence on $1/\delta$ in Eqns.~\eqref{eqn:expected_incentive} and \eqref{eqn:ucb_incentive} when $M$ is small.

At last, Fig~\ref{fig:incentive_M} reports the dependence of cumulative incentives on the number of agents. Under different $M$, random local instances with $30$ arms are generated to compose global instances with $\Delta_{\min}\in [4.5, 5.5]\times 10^{-3}$. As shown in Fig~\ref{fig:incentive_M}, the cumulative incentives (with $\delta = 0.01$) gradually diminish as $M$ increases. When more than $120$ agents are involved, the principal spends no incentive but can still learn the optimal global arm, which verifies our theoretical finding that with sufficiently many learning agents involved, the exploration process of the principal can be (almost) free.

\shir{
\section{Discussions and Future Works}\label{sec:future}
While progresses have been made in this work, some problems are worth future investigations.

{\bf Personalized Tasks.} This work focuses on identifying one common optimal global arm among the entire group of agents. This objective is well-motivated \citep{zhu2020federated} as only one arm can be selected for the collective interest in many applications. For example, due to the budget constraint, many companies must choose one out of multiple potential products for R\&D. However, we note that it is also an interesting direction to use global information for personalized tasks, e.g., recommend items to the customer based on both her own favor and the overall popularity. The idea of FMAB with personalization in \cite{shi2021federated2} is one potential direction, where weighted sums with the global and local model are set as the learning objectives.

{\bf Agent Sampling.} The principal in this work targets at learning the optimal global arm among the \emph{involved} group of agents, which is reasonable for many use cases. For example, cellular communication operators typically can access data from all users to find the optimal channel. Moreover, to provide products to its chain stores, the company can easily perform a market survey with all of them. 
Nevertheless, in some applications, the principal only has access to the feedback from a sampled population but nevertheless would like to learn the best arm for the underlying entire population. We believe our findings is an important step towards this more challenging setting, where certain probably approximately correct (PAC) learning style of analysis may be needed since one may have to bound the mismatch between the best arm for the involved group and that for the whole population as the number of involved agents grows.

{\bf Incentive-Provision Protocol.} As mentioned in Section~\ref{subsec:step1}, ``Take-or-Ban'' is for the purpose of rigorous theoretical analysis. Without this scheme, it is very challenging to rule out the possibility that some rational users may employ sophisticated strategies to intentionally reject an earlier incentive in order to induce higher future incentives. We note that it is intriguing to develop other agent regulation protocols or even directly analyze agents' behaviors without regulation. However, in reality, when facing less sophisticated users, we expect the insight revealed from our theoretical analysis still to be useful for less restrictive protocols, such as banning an agent for a few rounds or after a few (instead of one) refuses of the incentives (see Appendix~\ref{supp:exp} for some experimental illustrations). 

{\bf Other Extensions.} In Section~\ref{sec:theory}, a high-probability upper bounds of $C(T)$ is established for the agents who run $\alpha$-UCB. It is conceivable to extend the proof to other optimism-based algorithms, e.g., KL-UCB \citep{garivier2011kl}. However, it would be interesting to provide similar guarantees with agents running Thompson Sampling \citep{agrawal2012analysis} or $\epsilon$-greedy \citep{auer2002finite}. Furthermore, it is worth exploring how to extend the study to other bandit types.
}

\section{Conclusions}
In this work, we studied incentivized exploration with multiple long-term strategic agents. Motivated by practical applications, the formulated problem involves multiple heterogeneous agents aiming at collecting high cumulative local rewards and one principal trying to identify the optimal global arm but lacking direct accesses to the global model. The OTI algorithm was designed for the principal to intelligently leverage incentives to have local agents explore information on her behalf. Three key novel components played critical roles in the design and analysis of OTI: (1) a provably effective ``Take-or-Ban'' incentive-provision strategy to guarantee agents' behaviors; (2) a specifically designed best arm identification algorithm to aggregate local information of varying amounts provided by heterogeneous agents; (3) a high-probability lower-bound for UCB algorithms that proved its conservativeness. The regret analysis of OTI showed that the learning behaviors of strategic agents can provide ``free pulls'' to benefit the principal's exploration. Moreover, we observed that increasing the population of agents can also contribute to lower the burden of principal. At last, the key and somewhat surprising result was revealed that with sufficiently many learning agents involved, the exploration process of the principal can be (almost) free. 

\newpage

\bibliography{bandit,comb_bandit,incentive}
\bibliographystyle{apalike}


\newpage
\appendix

\section{Discussions of Duration of the Observing Phase}\label{supp:dissucsion}
In Section~\ref{subsec:observation}, the duration of the observing phase is specified as $\kappa(T)=\frac{T}{2}$, and we here discuss the influence of this choice and other available choices. On one hand, intuitively, if the observing phase lasts longer, more ``free pulls'' can be leveraged by the principal. On the other hand, there should be sufficient time reserved for adaptive adjustments, i.e., the incentivizing phase; otherwise, the principal cannot guarantee the success of best arm identification. It thus requires a careful trade-off between more ``free pulls'' and sufficient adaptation.

If the principal knows the parameter $\Delta_{\min}$ of the global game, i.e., the global sub-optimality gap, with Theorems~\ref{thm:expected_incentive} and \ref{thm:ucb_incentive}, she can specify $\kappa(T) = \kappa_o(T):=T-  \frac{16K\log(\frac{KT}{\delta})}{\Delta_{\min}^2}+\frac{64MK\log\log(\frac{KT}{\delta})}{\Delta_{\min}^2} = T- O(\log(T))$, which is an upper bound of the required number of pulls on local arms without incentives.  However, it is often impossible for principal to have such information. Since best arm identification is the primal task of the principal, we choose to specify $T - \kappa(T)$  to be $\omega(\log(T))$, i.e., with an order higher than $\log(T)$, to guarantee sufficient times are left for the incentivizing phase, and the adopted $\kappa(T) = \frac{T}{2}$ is an exemplary choice among others, e.g., $\frac{T}{4}, \sqrt{T}$. As shown in the later proofs for Theorems~\ref{thm:expected_incentive} and \ref{thm:ucb_incentive}, the amount of free pulls is of order $O(\log(\kappa(T)))$. Thus, while these choices ($\kappa_o(T), \frac{T}{2}, \frac{T}{4}, \sqrt{T}$, etc) seemingly distinct with each other, the amount of free pulls they provide does not differ much. 

In practical applications, it is also conceivable to perform estimation of $\Delta_{\min}$ during the game with $\hat{\mu}_k(t)$, i.e., $\hat{\Delta}_{\min}(t)$, and use the estimation to determine $\kappa(T)$. However, it is difficult to provide a rigorous theoretical analysis for such an adaptive approach.

\section{Proof of Lemma~\ref{lem:participate}}\label{supp:proof_start}
\begin{lemma}[Restatement of  Lemma~\ref{lem:participate}]
Under the  ``Take-or-Ban'' incentive-provision protocol, following incentives, whenever offered, is optimal in terms of the expected cumulative rewards (compared to not following) for every agent.
\end{lemma}
\begin{proof}
We fix an arbitrary agent $m$. Her local action $\pi_m(t)$ is made with the history $H_m(t) := \{\pi_m(\tau), X_{\pi_m(\tau),m}(\tau), I_{m}(\tau)|{\tau\leq t-1}\}$ and the current incentives $I_m(t)$, where $I_m(t): = \{I_{k,m}(t)|\forall k \in [K]\}$. Thus, we can write $\pi_m(t) = \Pi_m(H_m(t),I_m(t))$, where $\Pi_m$ is the strategy that maps the history and current incentives to actions.
    
 Key to our proof is to argue that if strategy $\Pi_m$ does not always follow the incentives, then another strategy $\Pi'_m$ which always follows the incentives whenever offered will do better in expectation. Formally,  $\Pi'_m$ is defined as follows based on a modified history $H'_m(t)$ and current incentive $I_m(t)$:
    \begin{itemize}
        \item When there is no incentive, strategy $\Pi'_m$ follows the decisions from $\Pi_m$ using the modified history $H'_m(t)$ and the observations are added to the modified history. Formally, if $\forall k\in[K], I_{k,m}(t)= 0$, then $\pi'_m(t)  = \Pi'_m(H_m(t), I_m(t)) = \Pi_m(H'_m(t), I_m(t))$ and observations $\{\pi'_m(t), X_{\pi'_m(t),m}(t),I_m(t)\}$ are added to $H'_m(t+1)$.  
        \item When there are incentives and $\Pi_m$ follows the incentive, $\Pi_m'$ also takes the incentive and adds observations to the modified history. Formally,  if there exists $k\in [K]$ such that $I_{k,m}(t)=1$ while $\Pi(H'_m(t), I_m(t))=k$, then $\pi'_m(t)  = \Pi'_m(H_m(t), I_m(t))=k$ and observations $\{\pi'_m(t), X_{\pi'_m(t),m}(t),I_m(t)\}$ are added to $H'_m(t+1)$. 
        \item When there are incentives but $\Pi_m$ does not take them, $\Pi_m'$ always takes the incentive, but importantly \emph{does not add observations to the modified history}. Formally, if there exists $k\in [K]$ such that $I_{k,m}(t)=1$ but $\Pi_m(H'_m(t), I_m(t))\neq k$, then $\pi'_m(t)  = \Pi'_m(H_m(t), I_m(t))=k$ and no changes are made to the modified history, i.e., $H'_m(t+1)=H'_m(t)$. 
    \end{itemize}
    
    If strategy $\Pi_m$ does not always take the incentives, there must be a time step $s$, $\exists k\in[K], I_{k,m}(s)=1$ but $\Pi_m(H_m(s))\neq k$.  After time $s$, the agent is banned from taking incentives any more. The expected cumulative reward of $\pi$ can thus be decomposed as
    \begin{equation*}
        \Eb[R^{\Pi_m}_m(T)] = \Eb\left[\sum_{t=1}^{s-1} (X_{\pi_m(t),m}(t)+I_{\pi_m(t),m}(t))+\sum_{t=s}^TX_{\pi_m(t),m}(t)\right].
    \end{equation*}

    With strategy $\Pi'_m$, for time step $t<s$, $H'_m(t), \pi_m'(t)$ are the same as $H_m(t),\pi_m(t)$. Thus, the cumulative reward of the designed $\Pi'_m$ can also be decomposed as
    \begin{align*}
        & \Eb[R^{\Pi_m'}_m(T)] \\
        = &\Eb\left[\sum_{t=1}^{T} (X_{\pi'_m(t),m}(t)+I_{\pi'_m(t),m}(t))\right]\\
        =  &\Eb\left[\sum_{t=1}^{s-1}(X_{\pi'_m(t),m}(t)+I_{\pi'_m(t),m}(t))+\sum_{t=s}^T(X_{\pi'_m(t),m}(t)+I_{\pi'_m(t),m}(t))\right]\\
        =&  \Eb\left[\sum_{t=1}^{s-1}(X_{\pi_m(t),m}(t)+I_{\pi_m(t),m}(t))+\sum_{t=s}^T(X_{\pi'_m(t),m}(t)+I_{\pi'_m(t),m}(t))\right]\\
        \geq &\Eb\left[\sum_{t=1}^{s-1}(X_{\pi_m(t),m}(t)+I_{\pi_m(t),m}(t))+\sum_{t\in[s,T]/\tau_m^{s,T}}X_{\pi'_m(t),m}(t)+|\tau_m^{s,T}|\right]
    \end{align*}
    where $\tau^{s,T}_m$ denotes the set of time slots that principal provides incentives in time interval $[s,T]$, i.e., $\tau_m^{s,T} = \{t\in[s,T]|\exists k\in[K], I_{k,m}(t)=1\}$.
    
    Since the observation from incentives are not counted in $H'_m(t)$, the distribution of $\{H'_m(t)|{t\in [s,T]/\tau_m^{s,T}}\}$ is the same as the distribution of $\{H_m(t)|{t\in [s,T-|\tau_m^{s,T}|]}\}$, which further means the distribution of $\{\pi'_m(t)|{t\in [s,T]/\tau_m^{s,T}}\}$ is the same with $\{\pi_m(t)|{t\in [s,T-|\tau_m^{s,T}|]}\}$. Thus, we can get
    \begin{align*}
         \Eb\left[\sum_{t\in[s,T]/\tau_m^{s,T}}X_{\pi'_m(t),m}(t)\right] =\Eb\left[\sum_{t\in[s,T-|\tau_m^{s,T}|]}X_{\pi_m(t),m}(t)\right].
    \end{align*}
    
    With this result, it holds that
    \begin{align*}
        & \Eb[R^{\Pi_m}_m(T)] \\
        = &\Eb\left[\sum_{t=1}^{s-1} (X_{\pi_m(t),m}(t)+I_{\pi_m(t),m}(t))+\sum_{t=s}^TX_{\pi_m(t),m}(t)\right]\\
        = &\Eb\left[\sum_{t=1}^{s-1} (X_{\pi_m(t),m}(t)+I_{\pi_m(t),m}(t))+\sum_{t=s}^{T-|\tau_m^{s,T}|}X_{\pi_m(t),m}(t)+\sum_{t=T-|\tau_m^{s,T}|+1}^TX_{\pi_m(t),m}(t)\right]\\
        \leq&  \Eb\left[\sum_{t=1}^{s-1}(X_{\pi_m(t),m}(t)+I_{\pi_m(t),m}(t))+\sum_{t\in[s,T]/\tau_m^{s,T}}X_{\pi'_m(t),m}(t)+|\tau_m^{s,T}|\right]\\
        \leq &  \Eb[R^{\Pi_m'}_m(T)].
    \end{align*}

    Thus, strategy $\Pi'_m$ always follows incentives and provides at least the same expected cumulative rewards as strategy $\Pi_m$, which does not always follow incentives. The lemma is thus proved.
\end{proof}
    

\section{Proof of Lemma~\ref{lem:confidence}}
\begin{lemma}[Restatement of Lemma~\ref{lem:confidence}]
    Denote 
    \begin{equation*}
        \Hc:=\left\{\forall t\in\left[\frac{T}{2}+1,T\right], \forall k\in [K], \left|\hat{\mu}_{k}(t-1)-\mu_k\right|\leq CB_k(t-1)\right\}.
    \end{equation*}
    When the horizon $T$ is sufficiently large, it holds that $\Pb(\mathcal{H})\geq 1-\delta$.
\end{lemma}

\begin{proof}


Using the Cauchy-Shwarz inequality, we have
\begin{align*}
    &\sum_{m\in[M]}N_{k,m}(t-1)\left(\mu_{k,m}-\hat{\mu}_{k,m}(t-1)\right)^2\leq\theta\\
    \Rightarrow & \sum_{m\in[M]}\frac{1}{N_{k,m}(t-1)}\sum_{m\in[M]}N_{k,m}(t-1)(\mu_{k,m}-\hat{\mu}_{k,m}(t-1))^2\leq \sum_{m\in[M]}\frac{\theta}{N_{k,m}(t-1)}\\
    \Rightarrow & \left(\sum_{m\in[M]}(\mu_{k,m}-\hat{\mu}_{k,m}(t-1))\right)^2\leq \sum_{m\in[M]}\frac{\theta}{N_{k,m}(t-1)}\\
    \Rightarrow & \frac{1}{M}\left|\sum_{m\in[M]}(\mu_{k,m}-\hat{\mu}_{k,m}(t-1))\right|\leq \frac{1}{M} \sqrt{\sum_{m\in[M]}\frac{\theta}{N_{k,m}(t-1)}}\\
    \Rightarrow & \left|\hat{\mu}_{k}(t-1)-\mu_k\right|\leq \frac{1}{M} \sqrt{\sum_{m\in[M]}\frac{\theta}{N_{k,m}(t-1)}}.
\end{align*}

Now with the critical concentration inequality given by Lemma~\ref{lem:concentration} presented in the following, and $\theta = \log\left(\frac{KT}{\delta}\right)+4M\log\log\left(\frac{KT}{\delta}\right)$, the above implication further indicates that
\begin{align*}
    &\Pb\left[\left|\hat{\mu}_{k}(t-1)-\mu_k\right|\leq CB_{k}(t-1)\right]\\
    \geq &\Pb\left[\sum_{m\in[M]}N_{k,m}(t-1)\left(\mu_{k,m}-\hat{\mu}_{k,m}(t-1)\right)^2\leq\theta\right]\\
    = &1-\Pb\left[\sum_{m\in[M]}N_{k,m}(t-1)\left(\mu_{k,m}-\hat{\mu}_{k,m}(t-1)\right)^2\geq\theta\right]\\
    \geq & 1-\underbrace{2e^{M+1}\left(\frac{2\left(\log\left(\frac{KT}{\delta}\right)+4M\log\log\left(\frac{KT}{\delta}\right)\right)^2\log(\frac{KT}{\delta})}{M}\right)^M\cdot\frac{1}{(\log(\frac{KT}{\delta}))^{4M}}}_{:= \text{term (a)}}\cdot \frac{\delta}{KT},
\end{align*}
where the last inequality is from Lemma~\ref{lem:concentration} and term (a) is of order $O(\frac{M^Me^{M}}{\log^M(KT/\delta)})$. Thus, when $T$ is sufficiently large, $\mathbb{P}\left[\left|\hat{\mu}_{k}(t-1)-\mu_k\right|\leq CB_{k}(t-1)\right]\geq 1-\frac{\delta}{KT}$. Finally, with a union bound over $t\in[T/2+1,T]$ and $k\in [K]$, the lemma can be proved.
\end{proof}

\begin{lemma}\label{lem:concentration}
For any $t \in [T]$, any $k\in [K]$, and any $\theta\geq M+1$, we have
\begin{align*}\small
    \mathbb{P}\left[\sum_{m\in[M]} N_{k,m}(t-1)\left(\mu_{k,m}-\hat{\mu}_{k,m}(t-1)\right)^2\geq \theta\right]
    \leq 2e^{M+1}\left(\frac{(\theta-1)\lceil\theta\log(t)\rceil}{M}\right)^Me^{-\theta}.
\end{align*}
\end{lemma}
\begin{proof}[Proof]
The proof follows the ideas from Theorem 2 in \citet{magureanu2014lipschitz} and Theorem 22 in \citet{perrault2020efficient}. To prove Lemma~\ref{lem:concentration}, it suffices to prove the following two inequalities:
\begin{align}
   & \mathbb{P}\left[\sum_{m\in[M]} N_{k,m}(t-1)\left((\mu_{k,m}-\hat{\mu}_{k,m}(t-1))^+\right)^2\geq \frac{\theta}{2}\right]
    \leq e^{M+1}\left(\frac{(\theta-1)\lceil\theta\log(t)\rceil}{M}\right)^Me^{-\theta};\label{eqn:confidence1}\\
    &\mathbb{P}\left[\sum_{m\in[M]} N_{k,m}(t-1)\left((\mu_{k,m}-\hat{\mu}_{k,m}(t-1))^-\right)^2\geq \frac{\theta}{2}\right]
    \leq e^{M+1}\left(\frac{(\theta-1)\lceil\theta\log(t)\rceil}{M}\right)^Me^{-\theta}\label{eqn:confidence2},
\end{align}
where $x^+=\max\{x,0\}$ and $x^-=\min\{x,0\}$.

We first focus on proving Eqn.~\eqref{eqn:confidence1} and the same techniques can be applied to derive Eqn.~\eqref{eqn:confidence2}. We fix some $\theta\geq M+1$, and define the desired event as:
\begin{equation*}
    \Af(t) := \left\{\sum_{m\in[M]}N_{k,m}(t-1)\left((\mu_{k,m}-\hat{\mu}_{k,m}(t-1))^+\right)^2\geq \frac{\theta}{2}\right\},
\end{equation*}
and a partition of all possible pulling times as:
\begin{equation*}
    \forall \mathbf{d}\in \Nb^M, \Bf_{\mathbf{d}}(t):= \bigcap_{m\in[M]}\left\{\left(\frac{\theta}{\theta-1}\right)^{d_m-1}\leq N_{k,m}(t-1)< \left(\frac{\theta}{\theta-1}\right)^{d_m}\right\}.
\end{equation*}

Since each number of pulls $N_{k,m}(t-1)$ for $m\in [M]$ is bounded by $t$, the number of possible $\mathbf{d}\in\mathbb{N}^M$ such that $\mathbb{P}(\Bf_{\mathbf{d}}(t))>0$ is bounded by $\left\lceil\frac{\log(t)}{\log(\theta/(\theta-1))}\right\rceil^M$. With the following Lemma~\ref{lem:partition} and a union bound, we can get
\begin{align*}
    \Pb(\Af(t))&\leq \sum_{\mathbf{d}}\Pb(\Af(t)\cap \Bf_{\mathbf{d}}(t))\\
    &\leq \left\lceil\frac{\log(t)}{\log(\theta/(\theta-1)}\right\rceil^M \left(\frac{(\theta-1)e}{M}\right)^M e^{1-\theta}\\
    &\leq e^{M+1}\left(\frac{(\theta-1)\lceil\theta\log(t)\rceil}{M}\right)^Me^{-\theta}
\end{align*}
where the last inequality is from $\log(\frac{\theta}{\theta-1}) = \log(1+\frac{1}{\theta-1})\geq \frac{1/(\theta-1)}{1+1/(\theta-1)} = \frac{1}{\theta}$.
\end{proof}

\begin{lemma}\label{lem:partition}
Let $\mathbf{d}\in \Nb^M$. Then $\Pb(\mathfrak{A}(t)\cap \mathfrak{B}_{\mathbf{d}}(t))\leq \left(\frac{(\theta-1)e}{M}\right)^M e^{1-\theta}$.
\end{lemma}
\begin{proof}
Let $\boldsymbol{\zeta}\in \Rb^M_+$. When events
\begin{equation*}
    \Bf_{\mathbf{d}}(t)=\bigcap_{m\in[M]}\left\{\left(\frac{\theta}{\theta-1}\right)^{d_m-1}\leq N_{k,m}(t-1)\leq \left(\frac{\theta}{\theta-1}\right)^{d_m}\right\}
\end{equation*}
and
\begin{equation*}
    \mathfrak{A}'(t) := \bigcap_{m\in[M]} \left\{N_{k,m}(t-1)((\mu_{k,m}-\hat{\mu}_{k,m}(t-1))^+)^2>\frac{\zeta_m}{2}\right\},
\end{equation*}
happen, $\forall m\in[M]$, it holds that
\begin{align*}
\mu_{k,m}-\hat{\mu}_{k,m}(t-1)>\sqrt{\frac{\zeta_m}{2N_{k,m}(t-1)}}\geq\varepsilon_m:= \sqrt{\frac{\zeta_m}{2(\theta/(\theta-1))^{d_m}}}.
\end{align*}

Thus,  the above events $\Af'(t)$ and $\Bf_{\mathbf{d}}(t)$ further imply
\begin{align*}
&\frac{\theta-1}{\theta}\sum_{m\in[M]}\zeta_m\\
=&\sum_{m\in[M]} 2\left(\frac{\theta}{\theta-1}\right)^{d_m-1}\varepsilon_m^2\\
 \leq& \sum_{m\in[M]} 2N_{k,m}(t-1) \varepsilon_m^2\\
  = &\sum_{m\in[M]} 4N_{k,m}(t-1)\varepsilon_m\times \varepsilon_m - \sum_{m\in[M]} N_{k,m}(t-1)\frac{1}{8}(4\varepsilon_m)^2\\
 \leq& \sum_{m\in[M]} 4N_{k,m}(t-1)\varepsilon_m (\mu_{k,m}-\hat{\mu}_{k,m}(t-1))- \sum_{m\in[M]} N_{k,m}(t-1)\frac{1}{8}(4\varepsilon_m)^2\\
  = &\sum_{m\in[M]}\sum_{\tau=1}^{t-1} 4\varepsilon_m\lb\{\pi_m(\tau) = k\}(\mu_{k,m}-X_{k,m}(\tau))- \sum_{m\in[M]}\sum_{\tau=1}^{t-1} \frac{1}{8}(4\varepsilon_m\lb\{\pi_m(\tau) = k\})^2\\
 \leq& \underbrace{\sum_{m\in[M]}\sum_{\tau=1}^{t-1} 4\varepsilon_m\lb\{\pi_m(\tau) = k\}(\mu_{k,m}-X_{k,m}(\tau))}_{:=\Cf_1(t)}\\
 &- \underbrace{\sum_{m\in[M]}\sum_{\tau=1}^{t-1}\log \Eb\left[\exp\left(4\varepsilon_m\lb\{\pi_m(\tau) = k\}(\mu_{k,m}-X_{k,m}(\tau))\right)\right]}_{:=\Cf_2(t)},
 \end{align*}
 where the last inequality is because $\mu_{k,m}-X_{k,m}(\tau)$ is $\frac{1}{2}$-sub-Gaussian and it holds that 
 \begin{equation*}
     \Eb\left[\exp(4\varepsilon_m\lb\{\pi_m(\tau) = k\}(\mu_{k,m}-X_{k,m}(\tau)))\right]\leq \exp\left(\frac{1}{8}(4\varepsilon_m\lb\{\pi_m(\tau) = k\})^2\right)
 \end{equation*}
 
With these results, we can further get
\begin{align*}
\Pb(\Af'(t)\cap\Bf_{\mathbf{d}}(t))\leq &\Pb\left[\frac{\theta-1}{\theta}\sum_{m\in[M]}\zeta_m\leq \Cf_1-\Cf_2\right]\\
\overset{(a)}{\leq}& \exp\left(-\frac{\theta-1}{\theta}\sum_{m\in[M]}\zeta_m\right)\mathbb{E}\left[\exp\left[\Cf_1(t)-\Cf_2(t)\right]\right]\\
 \overset{(b)}{=}& \exp\left(-\frac{\theta-1}{\theta}\sum_{m\in[M]}\zeta_m\right)
\end{align*}
where inequality (a) is the standard Markov inequality, and (b) is from simple algebraic multiplication.

Note that
\begin{equation*}
    \Pb(\Af'(t)\cap\Bf_{\mathbf{d}}(t))=\Pb\left[\bigcap_{m\in[M]}\left\{2\mathbb{I}\{\mathfrak{B}_{d}(t)\}N_{k,m}(t-1)\left((\mu_{k,m}-\hat{\mu}_{k,m}(t-1))^+\right)^2>\zeta_m\right\}\right]
\end{equation*}
and 
\begin{equation*}
    \Pb(\Af(t)\cap\Bf_{\mathbf{d}}(t))=\Pb\left[\sum_{m\in[M]}2\mathbb{I}\{\mathfrak{B}_{d}(t)\}N_{k,m}(t-1)\left((\mu_{k,m}-\hat{\mu}_{k,m}(t-1))^+\right)^2>\theta\right].
\end{equation*}
Thus, with $G=M$ and $a = \frac{\theta-1}{\theta}$ in the following Lemma~\ref{lem:multivariate}, we can finally have
\begin{align*}
\mathbb{P}(\mathfrak{A}(t)\cap \mathfrak{B}_{k}(t))\leq \left(\frac{(\theta-1)e}{M}\right)^M e^{1-\theta}.
\end{align*}
\end{proof}

\begin{lemma}[Lemma 8 from \citet{magureanu2014lipschitz}]\label{lem:multivariate}
Let $G\geq 2$, $a\geq 0$. Let $\mathbf{Z}\in \mathbb{R}^G$ be a random variable such that  $\forall \boldsymbol{\zeta}\in \mathbb{R}^G_+$
\begin{equation*}
    \mathbb{P}\left[\boldsymbol{Z}\geq \boldsymbol{\zeta}\right]\leq \exp\left[-a\sum_{g\in[G]}\zeta_g\right].
\end{equation*}
Then for $\theta\geq \frac{G}{a}$, we have
\begin{equation*}
    \mathbb{P}\left[\sum_{g\in[G]}{Z_g\geq \theta}\right]\leq \left(\frac{a\theta e}{G}\right)^Ge^{-a\theta}.
\end{equation*}

\end{lemma}

\section{Proof of Lemma~\ref{lem:good_event}}
\begin{lemma}[Restatement of Lemma~\ref{lem:good_event}]
    When event $\Hc$ happens, $\forall t\in [\frac{T}{2}+1,T]$, we have $k_*\in S(t)$, i.e., the optimal global arm would not be eliminated. Moreover, it suffices to eliminate arm $k\neq k_*$ at time $t$, i.e., $k\notin S(t)$, when
    \begin{equation*}\small
    \forall m\in[M], N_{k,m}(t-1), N_{k_*,m}(t-1)\geq \frac{16\log(KT/\delta)}{M\Delta_k^2}+\frac{64\log\log(KT/\delta)}{\Delta_k^2}.
    \end{equation*}
\end{lemma}

\begin{proof}
When event $\Hc$ defined in Lemma~\ref{lem:confidence} happens, $\forall t\in[\frac{T}{2}+1,T], \forall k\in S(t-1)$, it holds
\begin{align*}
    \hat{\mu}_{k_*}(t-1)+ CB_{k_*}(t-1)\geq \mu_*\geq \mu_k\geq \hat{\mu}_{k}(t-1)- CB_{k}(t-1).
\end{align*}
Thus, the optimal global arm would not be eliminated. 

Then, as indicated in Eqn.~\eqref{eqn:elimination}, when $\hat{\mu}_{k_*}(t)-CB_{k_*}(t)\geq \hat{\mu}_k+CB_k(t)$, arm $k\neq k_*$ is ensured to be eliminated from the active arm set. Further, we note that when
\begin{equation*}
    \forall m\in[M], N_{k,m}(t-1), N_{k_*,m}(t-1)\geq \frac{16\left(\log(\frac{KT}{\delta})+4M\log\log(\frac{KT}{\delta})\right)}{M\Delta_k^2},
\end{equation*}
it holds that
\begin{align*}
    \hat{\mu}_{k_*}(t-1)-CB_{k_*}(t-1)\geq \mu_{k_*}-2CB_{k_*}(t-1)\geq \mu_{k_*}-\frac{\Delta_k}{2};\\
    \hat{\mu}_{k}(t-1)+CB_{k}(t-1)\leq \mu_{k}+2CB_{k}(t-1)\leq \mu_{k}+\frac{\Delta_k}{2},
\end{align*}
which means it suffices to eliminate arm $k$.
\end{proof}

\section{Proof of Lemma~\ref{lem:ucb_lower}}
In this section, the proof of Lemma~\ref{lem:ucb_lower} is provided, Note that in the following proofs, we consider the standard bandit setting without incentives, i.e., the agent runs $\alpha$-UCB on her local bandit game. The proof presented here is largely inspired by \citet{rangi2021secure}.
\begin{lemma}\label{lem:supp_ucb_confidence}
    For horizon $\Lambda$, define event
    \begin{equation*}
        \Gc_{m}:=\left\{\forall k\in [K],\forall t\in \left[\frac{\Lambda}{2}+1, \frac{3\Lambda}{4}\right], |\hat{\mu}_{k,m}(t-1)-\mu_{k,m}|\leq \sqrt{\frac{3\log(t)}{2N_{k,m}(t-1)}}\right\}.
    \end{equation*}
    It holds that $\Pb[\Gc_{m}]\geq 1-\frac{2K}{\Lambda}$.
\end{lemma}
\begin{proof}
\begin{align*}
    \Pb(\bar{\Gc}_{m}) &= \Pb\left[\exists k\in [K],\exists t\in \left[\frac{\Lambda}{2}+1, \frac{3\Lambda}{4}\right], |\hat{\mu}_{k,m}(t-1)-\mu_{k,m}|> \sqrt{\frac{3\log(t)}{2N_{k,m}(t-1)}}\right]\\
    &\leq \sum_{k\in[K]}\sum_{t=\frac{\Lambda}{2}+1}^{\frac{3\Lambda}{4}}\Pb\left[|\hat{\mu}_{k,m}(t-1)-\mu_{k,m}|> \sqrt{\frac{3\log(t)}{2N_{k,m}(t-1)}}\right]\\
    &\leq \sum_{k\in[K]}\sum_{t=\frac{\Lambda}{2}+1}^{\frac{3\Lambda}{4}}\sum_{\tau = 1}^{t-1}\Pb\left[|\hat{\mu}_{k,m}(t-1)-\mu_{k,m}|> \sqrt{\frac{3\log(t)}{2N_{k,m}(t-1)}},N_{k,m}(t-1)=\tau\right]\\
    &\leq \sum_{k\in[K]}\sum_{t=\frac{\Lambda}{2}+1}^{\frac{3\Lambda}{4}}\sum_{\tau = 1}^{t-1}\Pb\left[|\hat{\mu}_{k,m}(t-1)-\mu_{k,m}|> \sqrt{\frac{3\log(t)}{2\tau}}, N_{k,m}(t-1)=\tau\right]\\
    &\leq \sum_{k\in[K]}\sum_{t=\frac{\Lambda}{2}+1}^{\frac{3\Lambda}{4}}\sum_{\tau = 1}^{t-1} 2\exp\left(-2\cdot\frac{3\log(t)}{2\tau}\cdot \tau\right)\\
    & = \sum_{k\in[K]}\sum_{t=\frac{\Lambda}{2}+1}^{\frac{3\Lambda}{4}}\frac{2}{t^2}\\
    &\leq \frac{2K}{\Lambda}.
\end{align*}
\end{proof}

\begin{lemma}[Restatement of Lemma~\ref{lem:ucb_lower}]
        When $\Lambda$ satisfies $\frac{\Lambda}{\log^2(\Lambda)}> \frac{4K(\alpha-3/2)^2}{\Delta_{\min,m}^4}$, the $\alpha$-UCB algorithm with $\alpha\geq \frac{3}{2}$ satisfies that
    \begin{equation}
         \Pb\bigg[\forall k\in [K],N^w_{k,m}(\Lambda)\geq \frac{(\sqrt{\alpha}-\sqrt{1.5})^2\log(\frac{\Lambda}{2})}{4\Delta_{k,m}^2}\bigg]\geq 1-\frac{2K}{\Lambda}.
    \end{equation} 
\end{lemma}

\begin{proof}
To ease the exposition, the superscript in $N_{k,m}^w(t)$ is omitted in this proof as $N_{k,m}(t)$, but note that this proof discusses the behavior of $\alpha$-UCB without incentives. For horizon $\Lambda$ satisfying $\frac{\Lambda}{\log^2(\Lambda)}> \frac{4K(\alpha-3/2)^2}{\Delta_{\min,m}^4}$, Lemma~\ref{lem:ucb_lower} indicates that event
\begin{equation*}
    \Ec_m:=\{\forall k\in[K],N_{k,m}(\Lambda)\geq F_{k,m}(\Lambda)\}
\end{equation*}
happens with a probability at least $1-\frac{2K}{\Lambda}$, where $F_{k,m}(\Lambda):=\frac{(\sqrt{\alpha}-\sqrt{1.5})^2\log(\frac{\Lambda}{2})}{4\Delta^2_{k,m}}$. To prove this lemma, it suffices to prove that $\Pb[\bar{\Ec}_m]\leq \frac{2K}{\Lambda}$.

With \begin{equation*}
    \Gc_{m}:=\left\{\forall k\in[K],\forall t\in\left[\frac{\Lambda}{2}+1, \frac{3\Lambda}{4}\right], |\hat{\mu}_{k,m}(t-1)-\mu_{k,m}|\leq \sqrt{\frac{3\log(t)}{2N_{k,m}(t-1)}}\right\}
\end{equation*} from Lemma~\ref{lem:supp_ucb_confidence}, we have that
\begin{equation*}
    \Pb\left[\Gc_{m}\right]\geq 1-\frac{2K}{\Lambda}.
\end{equation*}
Thus, it suffices to prove that with event $\Gc_{m}$ happening, event $\bar{\Ec}_{m}$ does not happen.

We prove it by contradiction. Assume that while event $\Gc_{m}$ happens, the event $\bar{\Ec}_{m}$ also
happens, which means there exists arm $k$ such that $N_{k,m}(\Lambda)\leq F_{k,m}(\Lambda)$. Then, for the interval $[\frac{\Lambda}{2}+1,\frac{3\Lambda}{4}]$, we divide it into $F_{k,m}(\Lambda)$ blocks, and each block has length $\frac{\Lambda}{4F_{k,m}(\Lambda)}$. With the pigeonhole principle, there must exist one block $[t_1,t_3]$, in which arm $k$ is not pulled, i.e., $N_{k,m}(t_3)=N_{k,m}(t_1-1)\leq F_{k,m}(\Lambda)$.

With event $G_{m}$ happening, for arm $k$, it holds that $\forall t\in[t_1,t_3]$, 
\begin{align*}
    &\hat{\mu}_{k,m}(t-1)+\sqrt{\frac{\alpha\log(t)}{N_{k,m}(t-1)}}\\
    =&\hat{\mu}_{k,m}(t-1)+\sqrt{\frac{3\log(t)}{2N_{k,m}(t-1)}}+\left(\sqrt{\alpha}-\sqrt{\frac{3}{2}}\right)\sqrt{\frac{\log(t)}{N_{k,m}(t-1)}}\\
    \geq& \mu_{k,m}+ \left(\sqrt{\alpha}-\sqrt{\frac{3}{2}}\right)\sqrt{\frac{\log(t)}{N_{k,m}(t-1)}}\\
    \geq& \mu_{k,m}+ \left(\sqrt{\alpha}-\sqrt{\frac{3}{2}}\right)\sqrt{\frac{\log(\frac{\Lambda}{2})}{F_{k,m}(\Lambda)}}\\
     = &\mu_{k,m}+2\Delta_{k,m}\\
    \geq&  \mu_{*,m}+\Delta_{k,m}.
\end{align*}

We then make the following claim that 
\begin{equation*}
    \forall j\in [K]/k, N_{j,m}(t_3)-N_{j,m}(t_1-1)\leq  N_{\max} := \frac{4(\sqrt{\alpha}+\sqrt{3/2})^2\log(\frac{3\Lambda}{4})}{\Delta_{k,m}^2}.
\end{equation*}

If this claim does not hold, then there exists arm $i\in[K]/k$ such that
\begin{equation*}
    N_{i,m}(t_3)-N_{i,m}(t_1-1)> N_{\max,m},
\end{equation*}
which further means there exists $t_2\in [t_1,t_3]$ such that
\begin{equation*}
    N_{i,m}(t_2-1)-N_{i,m}(t_1-1) = N_{\max}
\end{equation*}
and arm $i$ is pulled at time $t_2$.
For this arm $i$, at time $t_2$, with event $\Gc_m$, we have
\begin{align*}
    &\hat{\mu}_{i,m}(t_2-1)+\sqrt{\frac{\alpha\log(t_2)}{N_{i,m}(t_2-1)}}\\
    \leq& \mu_{i,m}+\sqrt{\frac{3\log(t_2)}{2N_{i,m}(t_2-1)}}+\sqrt{\frac{\alpha\log(t_2)}{N_{i,m}(t_2-1)}}\\
    \leq& \mu_{*,m}+\left(\sqrt{\alpha}+\sqrt{\frac{3}{2}}\right)\sqrt{\frac{\log(t_2)}{N_{i,m}(t_2-1)}}\\
    \leq& \mu_{*,m}+\left(\sqrt{\alpha}+\sqrt{\frac{3}{2}}\right)\sqrt{\frac{\log(\frac{3\Lambda}{4})}{N_{\max}}}\\
     =&\mu_{*,m}+\frac{\Delta_{k,m}}{2}.
\end{align*}

With the property proved above for arm $k$, i.e.,
\begin{equation*}
    \hat{\mu}_{k,m}(t_2-1)+\sqrt{\frac{\alpha\log(t_2)}{N_{k,m}(t_2-1)}}\geq \mu_{*,m}+\Delta_{k,m},
\end{equation*}
we can observe that arm $i$ cannot be pulled at time $t_2$, which leads to a contradiction and thus proves the claim.

Since arm $k$ is not pulled in $[t_1,t_3]$, other arms must be pulled sufficiently. Using the above claim, it must hold that
\begin{align*}
    &\sum_{j\in[K]/k}N_{j,m}(t_3)-N_{j,m}(t_1-1) = t_3-t_1+1\\
    \Rightarrow& (K-1)N_{\max}\geq t_3-t_1+1 = \frac{\Lambda}{4F_{k,m}(\Lambda)}\\
    \Rightarrow& N_{\max}\geq \frac{\Lambda}{4KF_{k,m}(\Lambda)}\\
    \Rightarrow &\frac{4(\sqrt{\alpha}+\sqrt{3/2})^2\log(\frac{3\Lambda}{4})}{\Delta_{k,m}^2} =  N_{\max}\geq \frac{\Lambda}{4KF_k(\Lambda)} = \frac{\Lambda}{4K}\frac{4\Delta_{k,m}^2}{(\sqrt{\alpha}-\sqrt{3/2})^2\log(\frac{\Lambda}{2})}\\
    \Rightarrow& \frac{\Lambda}{\log^2(\Lambda)}\leq \frac{4K(\alpha-3/2)^2}{\Delta_{k,m}^4}\leq \frac{4K(\alpha-3/2)^2}{\Delta_{\min,m}^4},
\end{align*}
which contradicts with the requirement for $\Lambda$ in Lemma~\ref{lem:ucb_lower}. This concludes the proof.
\end{proof}

\section{Proof of Theorem~\ref{thm:expected_incentive}}
\begin{theorem}[Restatement of Theorem~\ref{thm:expected_incentive}]
It is the best interest for every agent to always accept the incentivized explorations under the ``Take-or-Ban'' protocol. Moreover, if the agents' local strategy is consistent without incentives and the horizon $T$ is sufficiently large, the OTI algorithm satisfies that $\Pb[\hat{k}_*(T) = k_*]\geq 1-\delta$, and the expected cumulative incentives are bounded as
    \begin{equation*}
        \Eb[C(T)] = O\bigg( \sum_{k\in [K]}\sum_{m\in[M]}\bigg[\frac{\log(\frac{KT}{\delta})}{M\Delta_k^2}+\frac{\log\log(\frac{KT}{\delta})}{\Delta_k^2}- \min\bigg\{\frac{T}{2}, \frac{\log(\frac{T}{2})}{\textup{\texttt{KL}}(\mu_{k,m},\mu_{*,m})}\bigg\}\bigg]^+\bigg),
    \end{equation*}
where $x^+ := \max\{x,0\}$.
\end{theorem}

\begin{proof}
First, with Lemma~\ref{lem:participate}, always following the incentives provides higher expected cumulative rewards than not always following. Thus, it is the best interest for every agent to always accept the incentivized explorations under the ``Take-or-Ban'' protocol.

As shown in Lemma~\ref{lem:confidence}, event $\Hc$ happens with probability at least $1-\delta$. When event $\Hc$ happens, the optimal global arm would not be eliminated from the active arm set, which means the best arm identification succeeds as long as all other sub-optimal arms are eliminated. Thus,  it suffices to analyze how many incentives are needed to eliminate all other sub-optimal arms.

Conditioned on event $\Hc$, we make the following claim regarding the cumulative incentives:
\begin{equation*}
    \forall k\in[K], \forall m\in [M], C_{k,m}(T)\leq Z_{k,m}(T):= \left[\frac{16\left(\log(\frac{KT}{\delta})+4M\log\log(\frac{KT}{\delta})\right)}{M\Delta_k^2} - N^w_{k,m}(\frac{T}{2})\right]^+,
\end{equation*}
where $C_{k,m}(T) := \sum_{t=1}^TI_{k,m}(t)$ denotes the cumulative incentives on agent $m$'s arm $k$.

To prove this claim, we first assume that there exists an arm-agent pair, namely, $(k',m')$, such that $C_{k',m'}(T)> Z_{k',m'}(T)$. We assume $k'$ is not the optimal arm $k_*$ here, but the same analysis applies for $k_*$ with minor changes.
Thus, there must exist $t'\in [\frac{T}{2}+1, T]$ such that  $C_{k',m'}(t'-1)=Z_{k',m'}(T)$ while $I_{k',m'}(t')=1$. Equivalently, we have
\begin{equation*}
    N_{k',m'}(t'-1) \geq \frac{16\left(\log(\frac{KT}{\delta})+4M\log\log(\frac{KT}{\delta})\right)}{M\Delta_{k'}^2},
\end{equation*}
and agent $m'$ is incentivized to explore arm $k'$ at time $t'$, i.e., $\bar{k}(t') = k'$ and $\bar{m}(t')=m'$.

However, since $\bar{m}(t') = \argmin_{m\in[M]}N_{\bar{k}(t'),m}(t'-1)$, it holds that
\begin{equation*}
    \forall m\in [M], N_{k',m}(t'-1)\geq \frac{16\left(\log(\frac{KT}{\delta})+4M\log\log(\frac{KT}{\delta})\right)}{M\Delta_{k'}^2},
\end{equation*}
which means $CB_{k'}(t'-1)\leq \frac{\Delta_{k'}}{4}$. Since $\bar{k}(t') = \argmin_{k\in S(t)}CB_k(t'-1)$, it must have that
\begin{equation*}
    \forall k\in S(t-1), CB_{k}(t'-1)\leq \frac{\Delta_{k'}}{4}.
\end{equation*}

Thus, it raises a contradiction because
\begin{align*}
    \hat{\mu}_{k_*}(t'-1)-CB_{k_*}(t'-1)\geq \mu_{k_*}-2CB_{k_*}(t-1)\geq \mu_{k_*}-\frac{\Delta_{k'}}{2};\\
    \hat{\mu}_{k'}(t'-1)+CB_{k}(t'-1)\leq \mu_{k'}+2CB_{k'}(t'-1)\leq \mu_{k'}+\frac{\Delta_{k'}}{2},
\end{align*}
which means that arm $k'$ should have been eliminated and thus cannot be incentivized.

With the above claim proved, the expected cumulative incentives can be bounded as
\begin{align}
    \Eb[C(T)] &= \Eb\left[\sum_{k\in [K]}\sum_{m\in[M]} C_{k,m}(T)\right]\notag\\
    &\leq \Eb\left[\sum_{k\in [K]}\sum_{m\in[M]}\left[ \frac{16\left(\log(\frac{KT}{\delta})+4M\log\log(\frac{KT}{\delta})\right)}{M\Delta_k^2} - N^w_{k,m}(\frac{T}{2})\right]^+\right]\notag\\
    & = \sum_{k\in [K]}\sum_{m\in[M]}\left[ \frac{16\log(\frac{KT}{\delta})}{M\Delta_k^2}+\frac{64\log\log(\frac{KT}{\delta})}{\Delta_k^2} - \Eb\left[N^w_{k,m}(\frac{T}{2})\right]\right]^+.\label{eqn:incentive_upper_claim}
\end{align}

With Eqn.~\eqref{eqn:asy_lower} from \citet{Lai:1985}, if the agents' local strategies are consistent, with horizon $\Gamma$, $\forall k\neq k_{*,m}$, it holds that
\begin{equation*}
    \liminf_{\Gamma\to\infty} \frac{\Eb[N^w_{k,m}(\Gamma)]}{\log(\Gamma)}\geq \frac{1}{\texttt{KL}(\mu_{k,m},\mu_{*,m})},
\end{equation*}
which is also stated in Eqn.~\eqref{eqn:asy_lower}. Thus, there exists $\Gamma_0$ such that $\forall \Gamma>\Gamma_0$, it holds that $\Eb[N^w_{k,m}(\Gamma)]\geq \frac{\log(\Gamma)}{\texttt{KL}(\mu_{k,m},\mu_{*,m})}$. For $k_{*,m}$, since the local strategies are consistent, $\forall \psi>0$, $\Eb[N^w_{k_{*,m},m}(\Gamma)] \geq \Gamma- o(\Gamma^{\psi})$. Thus, it holds that $\forall k\in [K]$, $\forall \psi>0$,
\begin{equation}\label{eqn:expected_freepulls}
    \Eb\left[N^w_{k,m}(\frac{T}{2})\right] = \Omega\left(\min\left\{\frac{T}{2} , \frac{\log(\frac{T}{2})}{\texttt{KL}(\mu_{k,m},\mu_{*,m}) }\right\}\right),
\end{equation}
where the minimal takes care of $\texttt{KL}(\mu_{k,m},\mu_{*,m})=0$ for arm $k_*$. By plugging Eqn.~\eqref{eqn:expected_freepulls} into Eqn.~\eqref{eqn:incentive_upper_claim}, Theorem~\ref{thm:expected_incentive} is proved.
\end{proof}



\section{Proof of Theorem~\ref{thm:ucb_incentive}}\label{supp:proof_end}
\begin{theorem}[Restatement of Theorem~\ref{thm:ucb_incentive}]
While the agents run $\alpha$-UCB algorithms with $\alpha\geq \frac{3}{2}$ and the horizon $T$ is sufficiently large, the OTI algorithm satisfies that $\Pb[\hat{k}_*(T) = k_*]\geq 1-\delta$.
Moreover, it holds that
\begin{equation*}
    \Pb\bigg[C(T) = O\bigg( \sum_{k\in [K]}\sum_{m\in[M]}\bigg[\frac{\log(\frac{KT}{\delta})}{M\Delta_k^2}+\frac{\log\log(\frac{KT}{\delta})}{\Delta_k^2}- \frac{\alpha\log(\frac{T}{2})}{\Delta_{k,m}^2}\bigg]^+\bigg)\bigg]\geq 1-\frac{4MK}{T}.
\end{equation*}
\end{theorem}
\begin{proof}
Theorem~\ref{thm:ucb_incentive} can be proved by plugging the UCB lower bound in Lemma~\ref{lem:ucb_lower} (instead of Eqn.~\eqref{eqn:asy_lower}) into the above proof of Theorem~\ref{thm:expected_incentive}.
\end{proof}

\section{Experimental Details}\label{supp:exp}
The codes and instructions for the experiments \shir{are publicly available at \url{https://github.com/ShenGroup/Observe_then_Incentivize}}. The experiments are light in computation, and were all performed on a mainstream PC. A few details for the experimental setups are provided in this section. First, if there are more than one arm remaining active at horizon $T$, OTI should output the one with the largest sample mean. This approach takes care of the scenarios with an extremely small (or even zero) global sub-optimality gap. Second, we find that while the $O(\log\log(\frac{1}{\delta}))$ term in the confidence bound in Eqn.~\eqref{eqn:confidence} is required theoretically, it is not very helpful in practice and sometimes even degrades the overall performance. Thus, in the simulation of OTI, the confidence bound is specified as $CB_k(t-1) = \frac{1}{M}\sqrt{(\sum\nolimits_{m\in[M]}\frac{1}{N_{k,m}(t-1)})\log(KT/\delta)}$. It is also interesting for future works to see whether the confidence bound in Eqn.~\eqref{eqn:confidence} can be tightened so that the $O(\log\log(\frac{1}{\delta}))$ term can be removed theoretically. 

\begin{wrapfigure}{R}{0.4\textwidth}	
	\centering
	\includegraphics[width=0.4\textwidth]{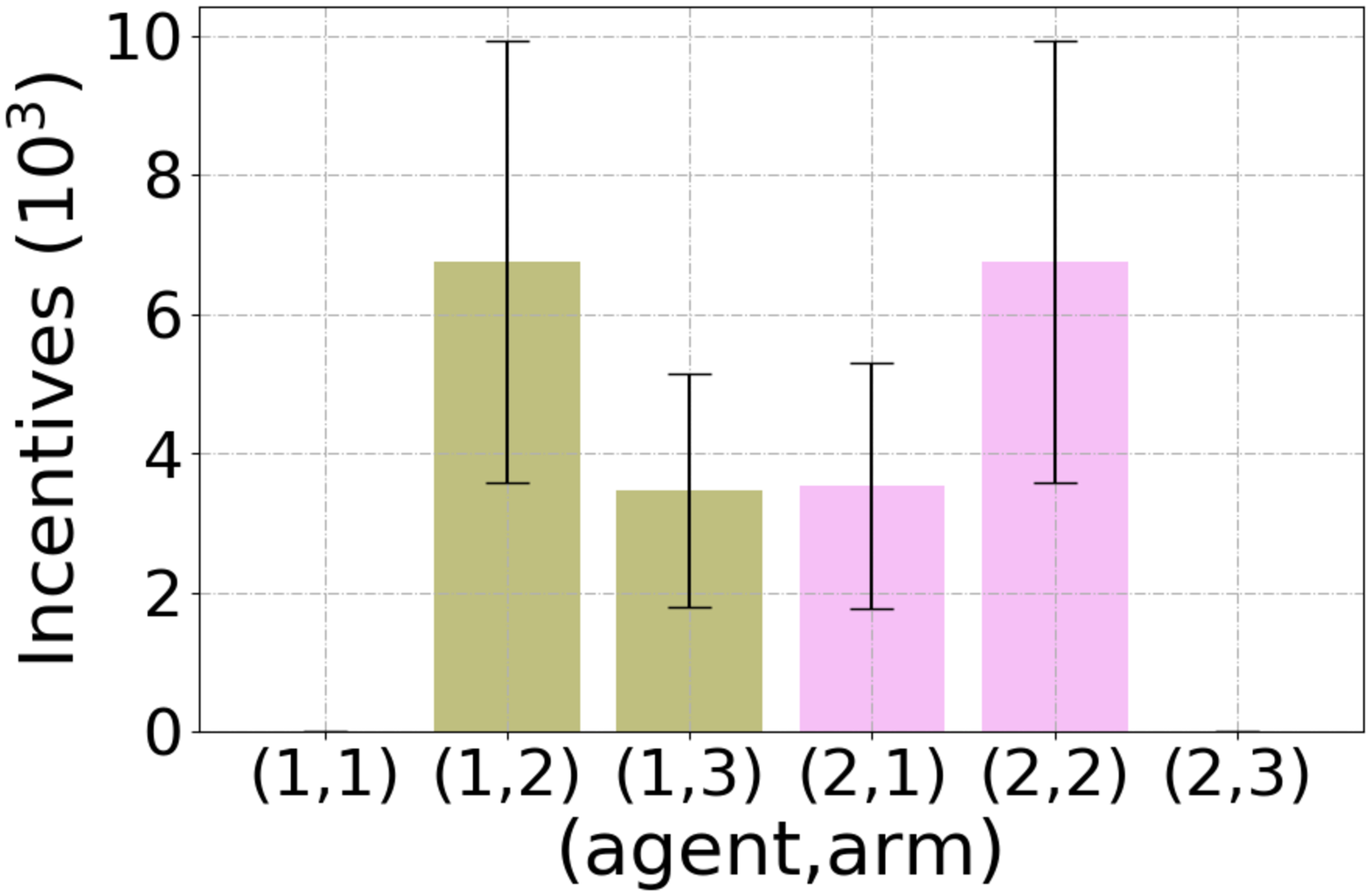}
	\caption{OTI with stochastic agent behaviors.}
	\label{fig:incentive_distribution_random}
\end{wrapfigure}

For all the simulations in Section~\ref{sec:exp}, the rewards are set to follow Bernoulli distributions. Futhermore, in the experiments of Fig.~\ref{fig:incentive_M}, the local game instances are chosen with the following schemes to have meaningful comparisons with different number of involving agents. First, a mean vector $\nu$ with $30$ arms is specified to be linearly distributed in $[0.4, 0.545]$, i.e., with gaps $0.005$. Then, for each arm $k\in[K]$, the mean $\mu_{k,m}$ of each player $m\in [M]$ is set as a sample from a truncated Gaussian distribution between $0$ and $1$ with mean $\nu_k$ and variance $0.01$. After this random sampling process, the local games are chosen. Then, if the corresponding global game has a sub-optimal gap $\Delta_{\min}\in [4.5, 5.5]\times 10^{-3}$, this instance is adopted; otherwise, a new instance is generated. This approach avoids the scenarios that with the number of involving agents increasing, the global game becomes more and more uniform, which makes comparisons with different $M$ unfair. 

\shir{Finally, additional experiments are performed to show that when dealing with relatively simple agents, OTI can be implemented with less restrictive incentive-provision protocols. In other words, the ``Take-or-Ban'' protocol is for theoretical rigor but may not be necessary in applications. Specifically, the agents are set to take the incentives with probability $0.8$ and refuse with probability $0.2$. Also, the principal never bans the agent regardless of their behaviors. Using the same game instance as in Fig.~\ref{fig:no_incentive}, we note that with such stochastic agent behaviors, OTI can still always identify the global optimal arm correctly. The spent incentives are shown in Fig.~\ref{fig:incentive_distribution_random}, which even slightly improve the performance in Fig.~\ref{fig:incentive_distribution}. This result also illustrates the robustness of OTI.}

\end{document}